\documentclass[10pt]{article}
\usepackage[margin=1in]{geometry}

\pdfoutput=1

\usepackage[colorlinks,citecolor=red]{hyperref}

\usepackage{graphicx}
\graphicspath{{figures/}}
\DeclareGraphicsExtensions{.pdf}
\usepackage[tight,footnotesize]{subfigure}

\usepackage{natbib}
\bibpunct{(}{)}{;}{a}{}{,}

\usepackage[cmex10]{amsmath}
\usepackage{amsfonts}
\usepackage{bm}
\numberwithin{equation}{section}

\usepackage{amsmath}
\usepackage{amsfonts}
\usepackage{mathrsfs}
\usepackage{amssymb}
\usepackage{color}

\usepackage{graphicx}
\usepackage{url}
\usepackage{algorithm}
\usepackage{algorithmic}
\usepackage{array}

\usepackage{amsthm}

\newtheorem{thm}{Theorem}

\usepackage{multirow}
\usepackage{array}
\renewcommand{\arraystretch}{1.5}

\begin{document}

\title{\textbf{Tensor Canonical Correlation Analysis for Multi-view Dimension Reduction}}

\author{Yong~Luo\thanks{Key Laboratory of Machine Perception (Ministry of Education), School of Electronics Engineering and Computer Science, Peking University, Beijing, China.}\ \thanks{Division of Networks and Distributed Systems, School of Computer Engineering, Nanyang Technological University, Singapore.}\ \thanks{Centre for Quantum Computation \& Intelligent Systems and the Faculty of Engineering \& Information Technology, University of Technology, Sydney, Sydney, Australia.}
\hspace{1cm} Dacheng~Tao\footnotemark[3]
\hspace{1cm} Yonggang~Wen\footnotemark[2] \\
Kotagiri~Ramamohanarao\thanks{Department of Computer Science and Software Engineering, The University of Melbourne, Australia.}
\hspace{1cm} Chao~Xu\footnotemark[1] \\
yluo180@gmail.com, dacheng.tao@uts.edu.au, ygwen@ntu.edu.sg \\ rao@csse.unimelb.edu.au, xuchao@cis.pku.edu.cn}


\date{08 February 2015}

\maketitle

\begin{abstract}
Canonical correlation analysis (CCA) has proven an effective tool for two-view dimension reduction due to its profound theoretical foundation and success in practical applications. In respect of multi-view learning, however, it is limited by its capability of only handling data represented by two-view features, while in many real-world applications, the number of views is frequently many more. Although the ad hoc way of simultaneously exploring all possible pairs of features can numerically deal with multi-view data, it ignores the high order statistics (correlation information) which can only be discovered by simultaneously exploring all features.

Therefore, in this work, we develop tensor CCA (TCCA) which straightforwardly yet naturally generalizes CCA to handle the data of an arbitrary number of views by analyzing the covariance tensor of the different views. TCCA aims to directly maximize the canonical correlation of multiple (more than two) views. Crucially, we prove that the multi-view canonical correlation maximization problem is equivalent to finding the best rank-1 approximation of the data covariance tensor, which can be solved efficiently using the well-known alternating least squares (ALS) algorithm. As a consequence, the high order correlation information contained in the different views is explored and thus a more reliable common subspace shared by all features can be obtained. In addition, a non-linear extension of TCCA is presented. Experiments on various challenge tasks, including large scale biometric structure prediction, internet advertisement classification and web image annotation, demonstrate the effectiveness of the proposed method.
\end{abstract}


\section{Introduction}\label{sec:introduction}

The features utilized in many real-world data mining tasks are frequently of high dimension and extracted from multiple views (or sources). For example, both the page content and hyperlink represented by bag-of-words (BOW) are usually used in web page classification \cite{A-Blum-and-T-Mitchell-COLT-1998, DP-Foster-et-al-TR-TTI-2008}, and it is common to combine the global (such as GIST \cite{A-Oliva-and-A-Torralba-IJCV-2001}) and local (such as SIFT \cite{DG-Lowe-IJCV-2004}) descriptors in image annotation \cite{TS-Chua-et-al-CIVR-2009, M-Guillaumin-et-al-ICCV-2009}. In these applications, the features can have dimensions of up to several hundred or thousand.

Multi-view dimension reduction \cite{DP-Foster-et-al-TR-TTI-2008} seeks a low-dimensional common subspace to compactly represent the heterogeneous data, in which each of the data examples is associated with multiple high-dimensional features. It often benefits the subsequent learning process significantly in that the curse-of-dimensionality is alleviated and the computation-al efficiency is improved \cite{CP-Hou-et-al-PR-2010, YH-Han-et-al-TCSVT-2012}. Canonical correlation analysis (CCA), which is designed to inspect the linear relationship between two sets of variables \cite{DR-Hardoon-et-al-NCn-2004, FR-Bach-and-MI-Jordan-TR-Berkeley-2005}, was formally introduced as a multi-view dimension reduction method in \cite{DP-Foster-et-al-TR-TTI-2008}, where the authors prove that the labeled instance complexity can be effectively reduced under certain weak assumptions. In addition, CCA has been widely used for multi-view classification \cite{JDR-Farquhar-et-al-NIPS-2005}, regression \cite{SM-Kakade-and-DP-Foster-COLT-2007}, clustering \cite{MB-Blaschko-and-CH-Lampert-CVPR-2008, K-Chaudhuri-et-al-ICML-2009}, etc. Theoretically, Bach and Jordan \cite{FR-Bach-and-MI-Jordan-TR-Berkeley-2005} interpreted CCA probabilistically as a latent variable model, and thus it is able to be involved in a larger probabilistic model.

In spite of the profound theoretical foundation and practical success of CCA in multi-view learning, it can only handle data that is represented by two-view features. The features utilized in many real-world applications, however, are usually extracted from more than two views. For example, different kinds of color, texture and shape features are popular used in visual analysis-based tasks such as image annotation and video retrieval. A typical approach for generalizing CCA to several views is to maximize the sum of pairwise correlations between different views \cite{J-Via-et-al-NN-2007}. The main drawback of this strategy is that only the statistics (correlation information) between pairs of features is explored, while high-order statistics that can only be obtained by simultaneously examining all features is ignored.

\begin{figure}[!t]
\centering
\includegraphics[width=0.6\linewidth]{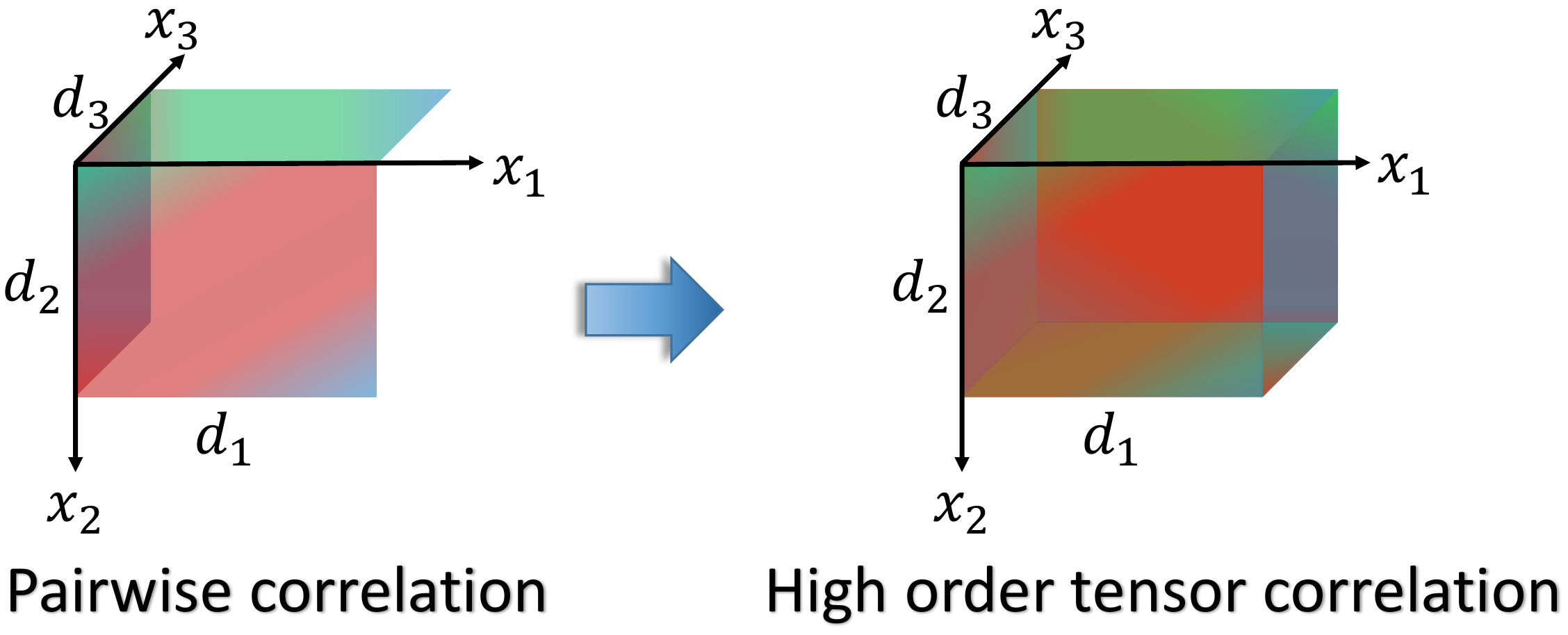}
\caption{The tensor CCA motivation. Only the pairwise correlation is explored in the traditional extensions of CCA, while much more information (i.e., the high order correlation) that can only be obtained by simultaneously examining all views is explored in the proposed TCCA.}
\label{fig:TCCA_Motivation}
\end{figure}

To tackle this problem, we develop tensor CCA (TCCA) to generalize CCA to handle an arbitrary number of views in a straightforward and yet natural way. In particular, TCCA aims to directly maximize the correlation between the canonical variables of all views, and this is achieved by analyzing the high-order covariance tensor over the data from all views. We prove that maximizing the correlation is equivalent to approximating the covariance tensor with a rank-1 tensor in an optimal least square sense. This approximation has been investigated in the literature and an efficient alternating least square (ALS) algorithm can be adopted for optimization \cite{PM-Kroonenberg-and-J-De-Leeuw-Psychometrika-1980, L-De-Lathauwer-et-al-HOPM-SIAM-2000, P-Comon-et-al-JoC-2009}. With respect to the traditional pairwise correlation maximization, the statistics (correlation information) explored can be measured using the $m(m-1)/2$ covariance matrices of size $¦¯(d^2)$, where $m$ is the number of views and $d$ represents the average feature dimensions, whereas in the proposed TCCA, the size of the covariance tensor is $¦¯(d^m)$. Fig. \ref{fig:TCCA_Motivation} is an illustrative example, where $m=3$. Much more correlation information is encoded in the common subspace shared by all features in multi-view dimension reduction, and thus hopefully better performance can be achieved. Furthermore, we extend the proposed TCCA to the non-linear case, which is useful when the feature dimensions are very high and limited instances are available. We perform extensive experiments on a variety of challenge tasks, including large scale biometric structure prediction, internet advertisement classification and web image annotation. We compare the proposed method with the traditional CCA \cite{DP-Foster-et-al-TR-TTI-2008} and its multi-view extension \cite{J-Via-et-al-NN-2007}, as well as two representative unsupervised multi-view dimension reduction approaches \cite{B-Long-et-al-SDM-2008, YH-Han-et-al-TCSVT-2012}. The results confirm the effectiveness of the proposed TCCA.

The article is organized as follows. We summarize closely related works in Section \ref{sec:Related_Work}. A brief introduction of CCA and its traditional multi-view extension is presented in Section \ref{sec:CCA_Generalization}. Section \ref{sec:TCCA} includes the description, formulation, and analysis of the proposed TCCA, as well as its non-linear extension kernel TCCA (KTCCA) for multi-view dimension reduction. Extensive experiments are presented in Section \ref{sec:Experiments} and the paper is concluded in Section \ref{sec:Conclusion}.

\section{Related Work}\label{sec:Related_Work}

\subsection{Multi-view Dimension Reduction}

Dimension reduction is a key technique in machine learning. The goal of dimension reduction is to find a low dimensional representation for high dimensional data \cite{T-Xia-et-al-TSMCB-2010}. Feature selection and feature transformation and the two main approaches for dimension reduction. The former aims to select a subset of variables from the original, while the latter transforms the data to a new space of fewer dimensions. The dimension reduction can be performed in an either unsupervised (e.g., principal component analysis (PCA) and Laplacian eigenmaps (LE) \cite{M-Belkin-and-P-Niyogi-NIPS-2001}), semi-supervised \cite{K-Benabdeslem-and-M-Hindawi-TKDE-2014}, or supervised (e.g., linear discriminant analysis (LDA)) setting, differed in the amount of labeled information being utilized.

In another research line, multi-view learning has attracted much attention recently. The multi-view we refer to here is the multiple feature representations of an object, not the spatial viewpoints in some other vision and graphics applications \cite{H-Su-et-al-ICCV-2009}. We generally classify the multi-view learning algorithms into three families: weighted view combination \cite{G-Lanckriet-et-al-JMLR-2004, B-Mcfee-and-G-Lanckriet-JMLR-2011}, multi-view dimension reduction \cite{DR-Hardoon-et-al-NCn-2004, M-White-et-al-NIPS-2012}, and view agreement exploration \cite{A-Blum-and-T-Mitchell-COLT-1998, A-Kumar-et-al-NIPS-2011}. Multi-view dimension reduction focuses on removing irrelevant or redundant information \cite{K-Benabdeslem-and-M-Hindawi-TKDE-2014} and reducing the feature dimension of data that consists of multiple views by leveraging the dependencies, coherence, and complementarity of those views. The different views are often assumed to be conditionally independent, thus a latent representation shared by all views can be obtained by exploiting the conditional independence structure of the multi-view data \cite{DP-Foster-et-al-TR-TTI-2008, B-Long-et-al-SDM-2008, M-White-et-al-NIPS-2012, YH-Han-et-al-TCSVT-2012, N-Chen-et-al-TPAMI-2012}. For example, canonical correlation analysis (CCA) is employed for multi-view dimension reduction in \cite{DP-Foster-et-al-TR-TTI-2008} to exploit the underlying conditional independence and redundancy assumption in multi-view learning. A general unsupervised learning method is presented in \cite{B-Long-et-al-SDM-2008} for multi-view data, where a consensus representation is learned by first applying dimension reduction technique (such as spectral embedding \cite{M-Belkin-and-P-Niyogi-NIPS-2001}) on each view and then combining the results via matrix factorization. In \cite{YH-Han-et-al-TCSVT-2012}, the structured sparsity \cite{R-Jenatton-et-al-JMLR-2011} is enforced among the different views in the learning of low-dimension consensus representation, to allow information being shared across subsets of features adaptively. In contrast to unsupervised multi-view dimension reduction, the similarity/dissimilarity pairwise constraints are utilized in \cite{CP-Hou-et-al-PR-2010} for semi-supervised multi-view dimension reduction. In \cite{N-Chen-et-al-TPAMI-2012}, the supervising information is also incorporated in the learned latent shared subspace by the use of a large-margin latent Markov network. In these methods, local optimal subspace can usually be obtained. Therefore, White et al. \cite{M-White-et-al-NIPS-2012} proposed a convex formulation for learning a shared subspace of multiple sources. In the learned subspace, conditional independence constraints are enforced.

\subsection{Canonical Correlation Analysis and Its Extensions}

Canonical correlation analysis (CCA), originally proposed by Hotelling (1936), finds bases for two random variables (or sets of variables) so that the coordinates of the variable pairs projected on these bases are maximally correlated \cite{DR-Hardoon-et-al-NCn-2004}. Much success has been achieved by applying CCA to pattern recognition and data mining. For example, SVM-2K was proposed in \cite{JDR-Farquhar-et-al-NIPS-2005} for two-view classification. It combines kernel CCA and support vector machine (SVM) in a single optimization problem, and the authors prove that the Rademacher complexity of SVM-2K is significantly lower than the individual SVMs. Kakade and Foster \cite{SM-Kakade-and-DP-Foster-COLT-2007} presented a multi-view regression algorithm regularized with a norm that is derived by applying CCA on unlabeled data. The authors show that the intrinsic dimension of the regression problem with the induced norm can be characterized by the correlation coefficients obtained in CCA. Under the conditionally uncorrelated assumption, a simple and efficient subspace learning algorithm based on CCA was proposed in \cite{K-Chaudhuri-et-al-ICML-2009} for multi-view clustering. The algorithm was shown to work well under much weaker separation conditions than the previous clustering methods.

In addition to these applications, there have been dozens of developments for CCA, most of which concentrate on inspecting the relationship between two sets of tensors rather than vectors. For example, the classical CCA was extended in \cite{SH-Lee-and-S-Choi-SPL-2007} to 2D-CCA, which directly analyzes 2D images without reshaping them into vectors. Some of its extensions are local 2D-CCA \cite{HX-Wang-SPL-2010}, sparse 2D-CCA \cite{JJ-Yan-et-al-SPL-2012}, and multilinear CCA (MCCA) \cite{HP-Lu-IJCAI-2013}. Considering that the two high-order tensors to be studied may share multiple modes (e.g., the video volume data), Kim and Cipolla \cite{TK-Kim-and-R-Cipolla-TPAMI-2009} presented two architectures for tensor correlation maximization by applying canonical transformation on the non-shared modes. In this way, features that have a good balance between flexibility and descriptive power may be obtained. This method is also termed ``tensor CCA'' (TCCA), but is quite different from the approach proposed in this paper. The main difference lies in that the latter focuses on analyzing two high-order tensor data sets, while our objective is to analyze the high-order statistics among multiple vector data sets (views).

The most closely related works to our methods, as far as we are concerned, are the maximum variance CCA (CCA-MAXVAR) \cite{JR-Kettenring-Biometrika-1971} and an adaptive CCA algorithm termed CCA-LS \cite{J-Via-et-al-NN-2007}, which is based on least square (LS) regression. The CCA-MAXVAR algorithm is performed by weighted combination of the canonical variables (projected vectors) of all views to approximate a latent common representation. This approach requires costly singular value decomposition (SVD) for optimization and cannot be trained in an adaptive fashion. To avoid these drawbacks, Via et al. \cite{J-Via-et-al-NN-2007} reformulated CCA-MAXVAR as a set of coupled LS regression problems, which seeks to minimize the distance between each pair of canonical variables. The reformulation is proved to be equivalent to the original CCA-MAXVAR formulation, but is much more efficient and can be learned adaptively. Nevertheless, there is still a disadvantage to both CCA-LS and CCA-MAXVAR, namely that only the pairwise correlations are exploited, while the high order correlations between all views are ignored. We developed the following tensor CCA framework to rectify this shortcoming.

\section{Canonical Correlation Analysis (CCA) and Its Multi-view Generalization}\label{sec:CCA_Generalization}

This section briefly introduces standard canonical correlation analysis (CCA) and its traditional generalizations on several data sets \cite{JR-Kettenring-Biometrika-1971, J-Via-et-al-NN-2007}. Given two sets of column vectors $\mathbf{x}_{1n} \in \mathbb{R}^{d_1}$, $\mathbf{x}_{2n} \in \mathbb{R}^{d_2}$, $n=1,\ldots,N$. The objective of CCA is to find a pair of projections (usually called canonical vectors) $\mathbf{h}_1$, $\mathbf{h}_2$, such that correlations between the two vectors of canonical variables $\mathbf{z}_1 \in \mathbb{R}^N$ and $\mathbf{z}_2 \in \mathbb{R}^N$ with each $z_{1n}=\mathbf{x}_{1n}^T \mathbf{h}_1$, $z_{2n}=\mathbf{x}_{2n}^T \mathbf{h}_2$, are maximized. The optimization problem is thus given by
\begin{equation}
\label{eq:CCA_Formulation}
\mathop{\mathrm{argmax}}_{\mathbf{z}_1,\mathbf{z}_2} \rho = \mathrm{corr}(\mathbf{z}_1, \mathbf{z}_2) = \frac{\mathbf{h}_1^T C_{12} \mathbf{h}_2}{\sqrt{\mathbf{h}_1^T C_{11} \mathbf{h}_1 \mathbf{h}_2^T C_{22} \mathbf{h}_2}},
\end{equation}
where $C_{11}=X_1 X_1^T$, $C_{22}=X_2 X_2^T$ are data variance matrices, and $C_{12}=X_1 X_2^T$ is the covariance matrix. Here, $X_1 \in \mathbb{R}^{d_1 \times N}$ and $X_2 \in \mathbb{R}^{d_2 \times N}$ are the stacked data matrices. The optimization of problem (\ref{eq:CCA_Formulation}) leads to the main solution of CCA, and the remaining solutions are given by maximizing the same correlation under the constraint of being orthogonal to the previous solutions.

CCA-MAXVAR \cite{JR-Kettenring-Biometrika-1971} generalizes CCA to $m$ views. Suppose the data matrix for the $p$'th view is $X_p \in \mathbb{R}^{d_p \times N}$, then the optimization problem of CCA-MAXVAR for finding the canonical vectors $\{\mathbf{h}_p\}_{p=1}^m$ is
\begin{equation}
\label{eq:CCA_MAXVAR_Formulation}
\begin{split}
\mathop{\mathrm{argmin}}_{\mathbf{z}, \mathbf{a}, \{\mathbf{h}_p\}_{p=1}^m} & \frac{1}{m} \sum_{p=1}^m \|\mathbf{z} - \alpha_p \mathbf{z}_p\|_2^2, \\
\mathrm{s.t.} & \ \|\mathbf{z}_p\|_2 = 1,
\end{split}
\end{equation}
where $\mathbf{z}_p=X_p^T \mathbf{h}_p$ is the vector of canonical variables, $\mathbf{z}$ is the best possible one-dimensional PCA representation, and $\mathbf{a}=[\alpha_1,\ldots,\alpha_m]^T$ is the vector of combination weights. To avoid a trivial solution, an additional constraint such as $\|\alpha_p\|_2^2=m$ is enforced. The solutions of (\ref{eq:CCA_MAXVAR_Formulation}) can be obtained using the SVD of $X_p$. To develop an efficient and adaptive algorithm, Via et al. \cite{J-Via-et-al-NN-2007} reformulated (\ref{eq:CCA_MAXVAR_Formulation}) as
\begin{equation}
\label{eq:CCA_LS_Formulation}
\begin{split}
\mathop{\mathrm{argmin}}_{\mathbf{z}, \mathbf{a}, \{\mathbf{h}_p\}_{p=1}^m} & \frac{1}{2m(m-1)} \sum_{p,q=1}^m \|X_p^T \mathbf{h}_p - X_q^T \mathbf{h}_q\|_2^2, \\
\mathrm{s.t.} & \ \frac{1}{m} \sum_{p=1}^m \mathbf{h}_p^T C_{pp} \mathbf{h}_p = 1.
\end{split}
\end{equation}
The orthogonal constraint $(z^{(i)})^T z^{(j)}=0, i \neq j$ is imposed on the different solutions, which can be obtained by using an iterative algorithm based on LS regression \cite{J-Via-et-al-NN-2007}. Here, $\mathbf{z}^{(i)}=\frac{1}{m} \sum_{p=1}^m \mathbf{z}_p^{(i)}$ and $\mathbf{z}_p^{(i)}$ is a vector of canonical variables projected using the $i$'th canonical vector in the $p$'th view.

\begin{figure*}[!t]
\centering
\includegraphics[width=0.9\linewidth]{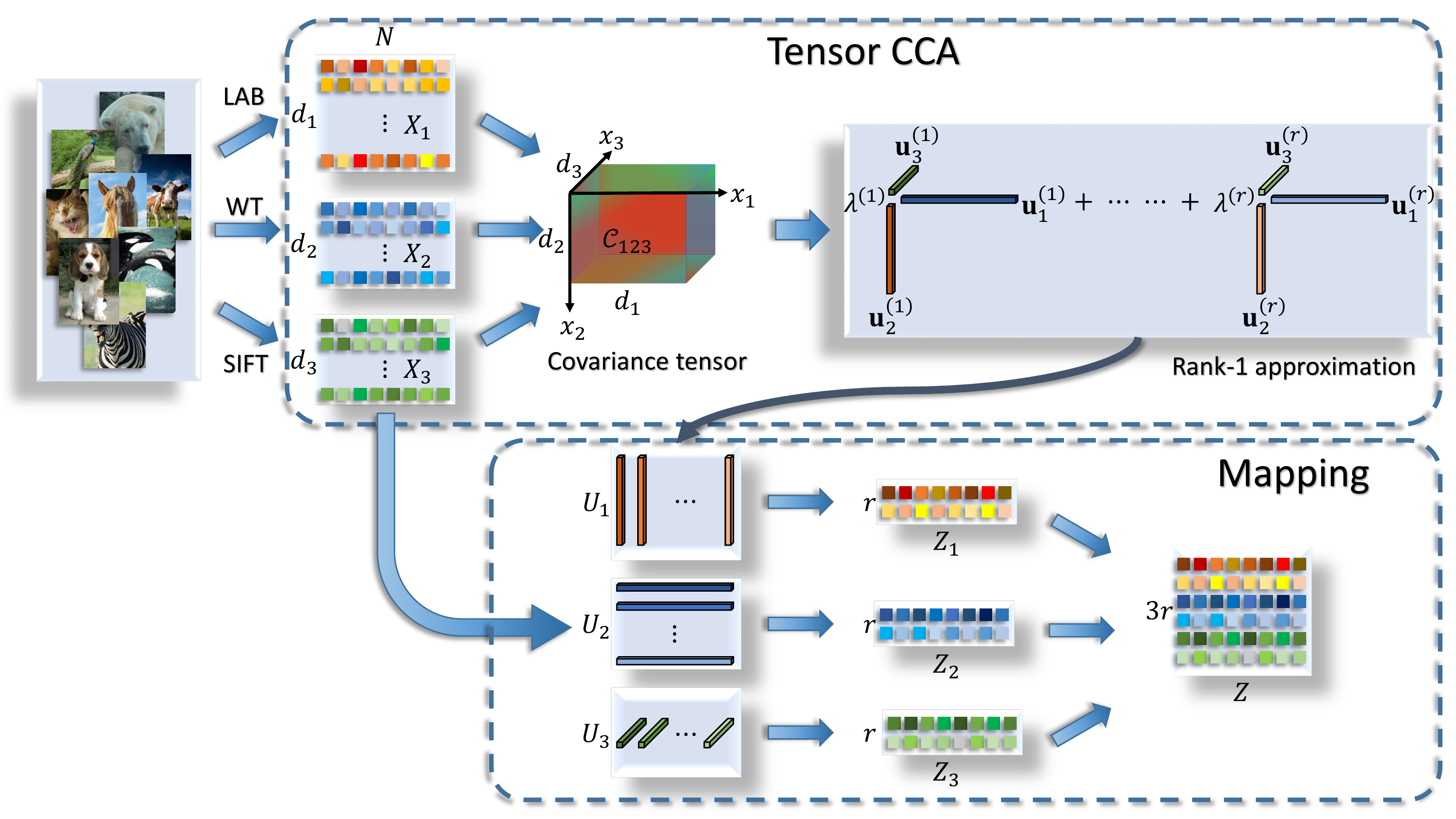}
\caption{System diagram of the multi-view dimension reduction method by the use of the proposed TCCA. Firstly, different kinds of features are extracted to represent the available instances in different views. Then a covariance tensor is calculated on the obtained representations $X_p, p = 1, \ldots ,m$ to discover the correlation information between all views. By approximating the covariance tensor with a set of rank-1 tensors, we obtain the transformation matrix $U_p$ for the $p$'th view. Each $U_p$ maps the original $X_p$ to the low dimensional $Z_p$ in the common subspace, and the final representation is a concatenation of $Z_p, p = 1, \ldots ,m.$}
\label{fig:System_Diagram}
\end{figure*}

\section{Tensor Canonical Correlation Analysis (TCCA)}\label{sec:TCCA}

In contrast to CCA-MAXVAR \cite{JR-Kettenring-Biometrika-1971} and CCA-LS \cite{J-Via-et-al-NN-2007}, where only the pairwise correlations are considered, we propose tensor CCA (TCCA) for multi-view dimension reduction by exploiting the high-order tensor correlation between all views. The diagram of the multi-view dimension reduction method using the proposed TCCA is shown in Fig. \ref{fig:System_Diagram}. Different kinds of features, such as LAB color histogram (LAB), wavelet texture (WT), and the local SIFT features (SIFT), are first extracted to represent the instances in different views. This leads to multiple feature matrices $\{ X_p \in \mathbb{R}^{d_p \times N} \}_{p=1}^m$. Here, $m$ is set at $3$ for intuitive illustration without loss of generality. The different sets of features are then used to calculate the data covariance tensor $C_{123}$, which is subsequently decomposed as a weighted sum of rank-1 tensors, i.e., $C_{123} \approx \sum_{k=1}^r \lambda^{(k)} \mathbf{u}_1^{(k)} \circ \mathbf{u}_2^{(k)} \circ \mathbf{u}_3^{(k)}$, where $r \leq \mathrm{min}(d_1,d_2,d_3)$ is the reduced dimension and $\circ$ is the tensor (outer) product. The vectors $\{ \mathbf{u}_p^{(k)} \}_{k=1}^r$ are stacked as a transformation matrix $U_p$, which is used to map the original high dimensional features into the low dimensional common subspace. The projected features $\{Z_p\}_{p=1}^m$ are concatenated as the final representation of the instances. The details of this technique are given below, but first we briefly introduce several useful notations and concepts of multilinear algebra.

\subsection{Notations}

Let $\mathcal{A}$ be an $m$-order tensor of size $I_1 \times I_2 \times \ldots \times I_m$, and $U$ be a $J_p \times I_p$ matrix. The $p$-mode product of $\mathcal{A}$ and $U$ is then denoted as $\mathcal{B}=\mathcal{A} \times_p U$, which is an $I_1 \times \ldots I_{p-1} \times J_p \times I_{p+1} \times I_m$ tensor with the element
\begin{equation}
\label{eq:TTM}
\begin{split}
& \mathcal{B}(i_1,\ldots,i_{p-1},j_p,i_{p+1},\ldots,i_m)
= \sum_{i_p=1}^{I_m} \mathcal{A}(i_1,i_2,\ldots,i_m) U(j_p,i_p).
\end{split}
\end{equation}
The product of $\mathcal{A}$ and a sequence of matrices $\{ U_p \in R^{J_p \times I_p} \}_{p=1}^m$ is a $J_1 \times J_2 \times \ldots \times J_m$ tensor denoted by
\begin{equation}
\label{eq:TTM_Sequence}
\mathcal{B}=\mathcal{A} \times_1 U_1 \times_2 U_2 \ldots \times_m U_m.
\end{equation}
The mode-$p$ matricization of $\mathcal{A}$ is denoted as an $I_p \times (I_1 I_{p-1} I_{p+1} I_m )$ matrix $A_{(p)}$, which is obtained by mapping the fibers associated with the $p$'th dimension of $\mathcal{A}$ as the rows of $A_{(p)}$, and aligning the corresponding fibers of all the other dimensions as the columns. Here, the columns can be ordered in any way. The $p$-mode multiplication $\mathcal{B} = \mathcal{A} \times_p U$ can be manipulated as matrix multiplication by storing the tensors in metricized form, i.e., $B_{(p)} = U A_{(p)}$. Specifically, the series of $p$-mode product in (\ref{eq:TTM_Sequence}) can be expressed as a series of Kronecker products and is given by
\begin{equation}
\label{eq:TTM_Sequence_Kronecker}
B_{(p)} = U_p A_{(p)} ( U^{(c_{p-1})} \otimes U^{(c_{p-2})} \otimes \ldots \otimes U^{(c_1)} )^T,
\end{equation}
where $\{ c_1,c_2,\ldots,c_L \} = \{ p+1,p+2,\ldots,m,1,2,\ldots,p-1 \}$ is a forward cyclic ordering for the indices of the tensor dimensions that map to the column of the matrix. Finally, the Frobenius norm of the tensor $\mathcal{A}$ is given by
\begin{equation}
\label{eq:Tensor_Frobenius_Norm}
\|\mathcal{A}\|_F^2 = \langle \mathcal{A},\mathcal{A} \rangle = \sum_{i_1=1}^{I_1} \sum_{i_2=1}^{I_2} \ldots \sum_{i_m=1}^{I_m} \mathcal{A}(i_1, i_2, \ldots, i_m)^2.
\end{equation}

\subsection{Problem Formulation}

Given $m$ views $\{X_p\}_{p=1}^m$ of $N$ instances, and each $X_p = [\mathbf{x}_{p1}, \mathbf{x}_{p2}, \ldots, \mathbf{x}_{pN}] \in \mathbb{R}^{d_p \times N}$ is assumed to have been centered (i.e., have zero mean). The variance matrices are then
\begin{equation}\notag
C_{pp} = \frac{1}{N} \sum_{n=1}^N \mathbf{x}_{pn} \mathbf{x}_{pn}^T, p = 1, \ldots, m,
\end{equation}
and the covariance tensor among all views is calculated as
\begin{equation}\notag
\mathcal{C}_{12 \ldots m} = \frac{1}{N} \sum_{n=1}^N \mathbf{x}_{1n} \circ \mathbf{x}_{2n} \circ \ldots \circ \mathbf{x}_{mn},
\end{equation}
where $\mathcal{C}$ is a tensor of dimension $d_1 \times d_2 \times \ldots \times d_m$. Following the objective of the traditional two-view CCA \cite{DR-Hardoon-et-al-NCn-2004}, the proposed tensor CCA seeks to maximize the correlation between the canonical variables $\mathbf{z}_p = X_p^T \mathbf{h}_p, p = 1,\ldots,m$, where $\{ \mathbf{h}_p \in \mathbb{R}^{d_p \times 1} \}_{p=1}^m$ are usually called the canonical vectors. Therefore, the optimization problem is
\begin{equation}
\label{eq:TCCA_Formulation}
\begin{split}
\mathop{\mathrm{argmax}}_{\{\mathbf{h}_p\}} \rho = & \mathrm{corr}(\mathbf{z}_1, \mathbf{z}_2, \ldots, \mathbf{z}_m) \\
\mathrm{s.t.} & \ \mathbf{z}_p^T \mathbf{z}_p = 1, p = 1, \ldots, m.
\end{split}
\end{equation}
Here $\mathrm{corr}( \mathbf{z}_1, \mathbf{z}_2, \ldots, \mathbf{z}_m ) = ( \mathbf{z}_1 \odot \mathbf{z}_2 \odot \ldots \odot \mathbf{z}_m )^T \mathbf{e}$ is the canonical correlation, and $\odot$ is the element-wise product, $\mathbf{e} \in \mathbb{R}^N$ is an all ones vector. We can prove that it is equivalent to $\mathcal{C}_{12 \ldots m} \times_1 \mathbf{h}_1^T \times_2 \mathbf{h}_2^T \ldots \times_m \mathbf{h}_m^T$, where $\times_p$ is the $p$-mode tensor-matrix product.
\begin{thm}
\label{thm:Canonical_Correlation}
The high order canonical correlation is given by
\begin{equation}
\label{eq:High_Order_Canonical_Correlation}
\rho = (\mathbf{z}_1 \odot \mathbf{z}_2 \odot \ldots \odot \mathbf{z}_m )^T \mathbf{e} = \mathcal{C}_{12 \ldots m} \times_1 \mathbf{h}_1^T \times_2 \mathbf{h}_2^T \ldots \times_m \mathbf{h}_m^T.
\end{equation}
\end{thm}
The proof is presented in the Appendix. By further considering that $X_p X_p^T = C_{pp}, p = 1, \ldots ,m$, the problem (\ref{eq:TCCA_Formulation}) becomes
\begin{equation}
\label{eq:TCCA_Reformulation}
\begin{split}
\mathop{\mathrm{argmax}}_{\{\mathbf{h}_p\}} \rho = & \mathcal{C}_{12 \ldots m} \times_1 \mathbf{h}_1^T \times_2 \mathbf{h}_2^T \ldots \times_m \mathbf{h}_m^T \\
\mathrm{s.t.} & \ \mathbf{h}_p^T C_{pp} \mathbf{h}_p = 1, p = 1, \ldots, m.
\end{split}
\end{equation}
We further add a regularization term in the constraints to control the model complexity, and thus the constraints of problem (\ref{eq:TCCA_Reformulation}) become
\begin{equation}
\mathbf{h}_p^T (C_{pp} + \epsilon I) \mathbf{h}_p = 1, p = 1, \ldots, m,
\end{equation}
where $I$ is an identity matrix and $\epsilon$ is a nonnegative trade-off parameter. Let each $\mathbf{u}_p = \tilde{C}_{pp}^{1/2} \mathbf{h}_p$ and $\mathcal{M} = \mathcal{C}_{12 \ldots m} \times_1 \tilde{C}_{11}^{-1/2} \times_2 \tilde{C}_{22}^{-1/2} \ldots \times_m \tilde{C}_{mm}^{-1/2}$, we can reformulate (\ref{eq:TCCA_Reformulation}) as
\begin{equation}
\label{eq:TCCA_Equivalent_Reformulation}
\begin{split}
\mathop{\mathrm{argmax}}_{\{\mathbf{u}_p\}} \rho = & \mathcal{M} \times_1 \mathbf{u}_1^T \times_2 \mathbf{u}_2^T \ldots \times_m \mathbf{u}_m^T \\
\mathrm{s.t.} & \ \mathbf{u}_p^T \mathbf{u}_p = 1, p = 1, \ldots, m,
\end{split}
\end{equation}
where $\tilde{C}_{pp} = C_{pp} + \epsilon I$. The equivalence of the problem (\ref{eq:TCCA_Reformulation}) and (\ref{eq:TCCA_Equivalent_Reformulation}) is ensured by the following theorem.
\begin{thm}
\label{thm:TCCA_Equivalent_Reformulation}
The problems (\ref{eq:TCCA_Reformulation}) and (\ref{eq:TCCA_Equivalent_Reformulation}) are equivalent.
\end{thm}
\begin{proof}
It is straightforward that the constraints of problems (\ref{eq:TCCA_Reformulation}) and (\ref{eq:TCCA_Equivalent_Reformulation}) are equivalent, and now we prove that the objective of the two problems is the same as follows,
\begin{equation}\notag
\begin{split}
& \mathcal{C}_{12 \ldots m} \times_1 \mathbf{h}_1^T \times_2 \mathbf{h}_2^T \times_m \mathbf{h}_m^T \\
= & \mathbf{h}_m^T C_{(m)} ( \mathbf{h}_{m-1} \otimes \ldots \otimes \mathbf{h}_2 \otimes \mathbf{h}_1 ) \\
= & \mathbf{u}_m^T \tilde{C}_{mm}^{-1/2} C_{(m)} ( (\tilde{C}_{m-1,m-1}^{-1/2} \mathbf{u}_{m-1}) \otimes \ldots \otimes (\tilde{C}_{11}^{-1/2} \mathbf{u}_1) ) \\
= & \mathbf{u}_m^T ( \tilde{C}_{mm}^{-1/2} C_{(m)} (\tilde{C}_{m-1,m-1}^{-1/2} \otimes \ldots \otimes \tilde{C}_{11}^{-1/2}) (\mathbf{u}_{m-1} \otimes \ldots \otimes \mathbf{u}_1) ) \\
= & \mathbf{u}_m^T \mathcal{M} ( \mathbf{u}_{m-1} \otimes \ldots \otimes \mathbf{u}_2 \otimes \mathbf{u}_1 ) \\
= & \mathcal{M} \times_1 \mathbf{u}_1^T \times_2 \mathbf{u}_2^T \ldots \times_m \mathbf{u}_m^T,
\end{split}
\end{equation}
where the metricizing property of the tensor-matrix product presented in (\ref{eq:TTM_Sequence_Kronecker}) and some basic properties of the Kronecker product are applied.
\end{proof}

\subsection{Solutions}

It has been presented in \cite{L-De-Lathauwer-et-al-HOPM-SIAM-2000} that the problem (\ref{eq:TCCA_Equivalent_Reformulation}) is equivalent to finding the best rank-$1$ approximation of the tensor $\mathcal{M}$, i.e., if we define $\hat{\mathcal{M}} = \rho \mathbf{u}_1 \circ \mathbf{u}_2 \circ \ldots \circ \mathbf{u}_m$, then the optimization problem becomes
\begin{equation}
\label{eq:TCCA_Optimization}
\mathop{\mathrm{argmin}}_{\{\mathbf{u}_p\}} \|\mathcal{M} - \hat{\mathcal{M}}\|_F^2.
\end{equation}
The solution can be obtained using the alternating least square (ALS) algorithm \cite{PM-Kroonenberg-and-J-De-Leeuw-Psychometrika-1980, P-Comon-et-al-JoC-2009}. Some other algorithms, such as the high-order power method (HOPM) \cite{L-De-Lathauwer-et-al-HOPM-SIAM-2000} and the tensor power method \cite{GI-Allen-AISTATS-2012}, can also be applied here for optimization, but our empirical findings indicate that the ALS algorithm performs the best in our experiments.

As in the two-view CCA, we perform a recursive maximization of the correlation between linear combinations of $X_p, p = 1, \ldots ,m$. However, we cannot expect the different linear combinations $\mathbf{h}_p^{(1)}, \ldots, \mathbf{h}_p^{(r)}$ of $X_p$ to be uncorrelated with each other, where $r$ is the rank of $\mathcal{M}$ (the determination of the rank value is still an open problem for the high-order tensors \cite{L-De-Lathauwer-et-al-HOSVD-SIAM-2000}). That is, the orthogonality constraints cannot be imposed on $\mathbf{u}_p^{(1)}, \ldots, \mathbf{u}_p^{(r)}$, since the sum of rank-1 decomposition and orthogonal decomposition of high-order tensors cannot be satisfied simultaneously \cite{L-De-Lathauwer-et-al-HOSVD-SIAM-2000}.

Based on the solutions $\mathbf{u}_p$, we obtain the canonical variables $\mathbf{z}_p = X_p^T \mathbf{h}_p = X_p^T \tilde{C}_{pp}^{-1/2} \mathbf{u}_p$. Let $U_p = [\mathbf{u}_p^{(1)}, \ldots, \mathbf{u}_p^{(r)}]$ and $\mathbf{z}_p^{(1)}, \ldots, \mathbf{z}_p^{(r)}$ be the column vectors of $Z_p$, we obtain the projected data for the $p$'th view:
\begin{equation}
\label{eq:Projected_Data}
Z_p = X_p^T \tilde{C}_{pp}^{-1/2} U_p.
\end{equation}
Following \cite{DP-Foster-et-al-TR-TTI-2008}, where it is suggested that the dimension be reduced to $2r$ in the standard CCA, we concatenate the different $\{ Z_p \}$ as the final representation $Z \in \mathbb{R}^{(mr) \times N}$ for the subsequent learning, such as classification \cite{JDR-Farquhar-et-al-NIPS-2005, D-Fisch-et-al-TKDE-2014}, clustering \cite{SM-Yang-et-al-TKDE-2014, JJ-Wu-et-al-TKDE-2015}, regression \cite{SM-Kakade-and-DP-Foster-COLT-2007}, search ranking \cite{B-Xu-et-al-TKDE-2015, HS-Zhu-et-al-TKDE-2015}, collaborative filtering \cite{Q-Liu-et-al-TKDE-2014}, and so on.

\subsection{Non-linear Extension}

The projections $\{ \mathbf{h}_p \}$ are linear in TCCA and thus may be not appropriate for instances that lie in quite non-linear feature space. To this end, we develop kernel tensor CCA (KTCCA) that extends the proposed TCCA to the non-linear case. KTCCA aims to find non-linear projections by first projecting the data into higher dimensional space induced by the feature mapping $\phi$:
\begin{equation}\notag
\phi(X_p) = [\phi(\mathbf{x}_{p1}), \phi(\mathbf{x}_{p2}), \ldots, \phi(\mathbf{x}_{pN})] \in \mathbb{R}^{D_p \times N},
\end{equation}
where the mapped dimension $D_p$ may be infinite. Then the variance matrices
\begin{equation}\notag
C_{pp} = \frac{1}{N} \sum_{n=1}^N \phi(\mathbf{x}_{pn}) \phi^T(\mathbf{x}_{pn}), p = 1, \ldots, m,
\end{equation}
the covariance matrix
\begin{equation}\notag
\mathcal{C}_{12 \ldots m} = \frac{1}{N} \sum_{n=1}^N \phi(\mathbf{x}_{1n}) \circ \phi(\mathbf{x}_{2n}) \circ \ldots \circ \phi(\mathbf{x}_{mn}),
\end{equation}
and the canonical variables $\mathbf{z}_p=\phi^T(X_p) \mathbf{h}_p$. It follows from the Representer Theorem \cite{B-Scholkopf-and-AJ-Smola-Book-MIT-2002} that $\mathbf{h}_p$ can be rewritten as a linear combination of the given instances, i.e.,
\begin{equation}
\label{eq:Representer_Theorem}
\mathbf{h}_p = \phi(X_p) \mathbf{a}_p,
\end{equation}
where $\mathbf{a}_p \in \mathbb{R}^{N \times 1}$ is a vector of the combination coefficients. The problem (\ref{eq:TCCA_Reformulation}) then becomes
\begin{equation}
\label{eq:KTCCA_Formulation}
\begin{split}
\mathop{\mathrm{argmax}}_{\{\mathbf{a}_p\}} \rho = & \mathcal{K}_{12 \ldots m} \times_1 \mathbf{a}_1^T \times_2 \mathbf{a}_2^T \ldots \times_m \mathbf{a}_m^T \\
\mathrm{s.t.} & \ \mathbf{a}_p^T K_{pp} \mathbf{h}_p = 1, p = 1, \ldots, m,
\end{split}
\end{equation}
where $K_{pp} = \phi^T(X_p) \phi(X_p)$ is the kernel matrix of the $p$'th view. The derivation is similar to Theorem \ref{thm:TCCA_Equivalent_Reformulation}. Here $\mathcal{K}_{12 \ldots m} = \mathcal{C}_{12 \ldots m} \times_1 \phi^T(X_1) \times_2 \phi^T(X_2) \ldots \times_m \phi^T(X_m)$ and can be calculated according to the following theorem.
\begin{thm}
\label{thm:Tensor_Kernel_Calculation}
The following equality holds:
\begin{equation}\notag
\begin{split}
& \mathcal{C}_{12 \ldots m} \times_1 \phi^T(X_1) \times_2 \phi^T(X_2) \ldots \times_m \phi^T(X_m)
= \frac{1}{N} \sum_{n=1}^N \mathbf{k}_{1n} \circ \mathbf{k}_{2n} \circ \ldots \circ \mathbf{k}_{mn},
\end{split}
\end{equation}
in which $\mathbf{k}_{pn} = \phi^T(X_p) \phi(\mathbf{x}_{pn})$, i.e., the $n$'th column of the kernel matrix $K_{pp}$, $p=1,\ldots,m$.
\end{thm}
We give the proof in the Appendix. To avoid trivial learning, we follow \cite{DR-Hardoon-et-al-NCn-2004} and introduce a partial least square (PLS) term to penalize the norms of the weight vectors $\{ \mathbf{a}_p \}$. That is, the constraints of problem (\ref{eq:KTCCA_Formulation}) become
\begin{equation}
\mathbf{a}_p^T (K_{pp}^2 + \epsilon K_{pp}) \mathbf{a}_p = 1, p = 1, \ldots, m.
\end{equation}
Because the matrix $(K_{pp}^2 + \epsilon K_{pp})$ is positive definite, it has a unique Cholesky decomposition, and we can denote its decomposition as $(K_{pp}^2 + \epsilon K_{pp}) = L_p^T L_p$. Let $\mathbf{b}_p = L_p \mathbf{a}_p$ and $\mathcal{S} = \mathcal{K}_{12 \ldots m} \times_1 (L_1^{-1})^T \times_2 (L_2^{-1})^T \ldots \times_m (L_m^{-1})^T$, we can reformulate (\ref{eq:KTCCA_Formulation}) as
\begin{equation}
\label{eq:KTCCA_Reformulation}
\begin{split}
\mathop{\mathrm{argmax}}_{\{\mathbf{b}_p\}} \rho = & \mathcal{S} \times_1 \mathbf{b}_1^T \times_2 \mathbf{b}_2^T \ldots \times_m \mathbf{b}_m^T \\
\mathrm{s.t.} & \ \mathbf{b}_p^T \mathbf{b}_p = 1, p = 1, \ldots, m,
\end{split}
\end{equation}
Similar to TCCA, this problem is equivalent to finding the best rank-$1$ approximation of $\mathcal{S}$, and the solution can be found using the ALS algorithm. By recursively maximizing the correlation, we obtain $\mathbf{b}_p^{(1)}, \ldots, \mathbf{b}_p^{(r)}$. Let $B_p = [\mathbf{b}_p^{(1)}, \ldots, \mathbf{b}_p^{(r)}]$, and the canonical variables $\mathbf{z}_p = \phi^T(X_p) \mathbf{h}_p = \phi^T(X_p) \phi(X_p) \mathbf{a}_p = K_{pp} L_p^{-1} \mathbf{b}_p$ and the projected data for the $p$'th view are then
\begin{equation}
\label{eq:NonLin_Projected_Data}
Z_p = K_pp L_p^{-1} B_p.
\end{equation}
The concatenated $Z \in \mathbb{R}^{(mr) \times N}$ is the final representation of the instances.

\subsection{Complexity Analysis}\label{subsec:Complexity_Analysis}

The time and space complexities of the proposed TCCA model are both closely related to the size of tensor $\mathcal{M}$. Straightforwardly, the space complexity is $O(d_1 d_2 \ldots d_m)$. Because the tensor $\mathcal{M}$ can be calculated offline, the time complexity is dominated by the rank-$r$ decomposition using the ALS algorithm. Considering that it is common that $r \ll \mathrm{min}(d_1, d_2, \ldots, d_m)$, we can speculate the time complexity of ALS is $O(t r d_1 d_2 \ldots d_m)$ according to \cite{P-Comon-et-al-JoC-2009}, where the time cost of the ALS algorithm for the three modes tensor is presented. Here, $t$ is the number of iterations in ALS.

According to the above analysis, we can see that the complexity of TCCA is independent of the number of instances, and thus our method can be scaled in very large sample size problems. Similarly, the complexities of KTCCA are determined by the tensor $\mathcal{S}$, the size of which is $N^m$. The space and time complexities are $O(N^m)$ and $O(trN^m)$ respectively. This means that KTCCA is capable of being scaled in problems that have very high feature dimensions and a small number of instances.

\section{Experiments}\label{sec:Experiments}

In this section, we empirically validate the effectiveness of the proposed TCCA on a biometric structure prediction and an advertisement classification problem following \cite{DP-Foster-et-al-TR-TTI-2008}, as well as on a challenging web image annotation task \cite{TS-Chua-et-al-CIVR-2009}. In all of the following experiments, five random choices of the labeled instances are used. Twenty percent of the test data (or unlabeled data in the transductive setting) are used for validation, which means that the parameters (if not specified) corresponding to the best performance on the validation set are used for testing. The evaluation criterion is the classification accuracy.

\subsection{Evaluation of the Linear Formulation}

In the first two sets of experiments (biometric structure prediction and advertisement classification), we use regularized least squares (RLS) as the base learner following \cite{DP-Foster-et-al-TR-TTI-2008}. Given $N_l$ labeled instances $\{ (\mathbf{x}_n,y_n) \}_{n=1}^{N_l}$, the optimization problem for RLS is given by $\mathop{\mathrm{argmin}}_\mathbf{w} \frac{1}{N_l} \sum_{n=1}^{N_l} (\mathbf{w}^T \mathbf{x}_n - y_n)^2 + \gamma \|\mathbf{w}\|_2^2$, where the positive trade-off parameter $\gamma$ is is set as $10^{-2}$ according to \cite{DP-Foster-et-al-TR-TTI-2008}. 
A constant feature of $1$ is appended to each instance to include a bias term in $\mathbf{w}$. In web image annotation, the $k$-nearest-neighbor ($k$NN) classifier is utilized, where the candidate set for $k$ is $\{ 1,2,\ldots,10 \}$. Specifically, we compare the following methods:
\begin{itemize}
  \item \textbf{BSF:} using the single view feature that achieves the best performance in RLS/$k$NN-based classification.
  \item \textbf{CAT:} concatenating the normalized features of all the views into a long vector, and then performing RLS/$k$NN-based classification.
  \item \textbf{CCA \cite{DP-Foster-et-al-TR-TTI-2008}:} using the CCA formulation presented in \cite{DP-Foster-et-al-TR-TTI-2008} to find a common representation of two different views. In this formulation, a regularization term $\epsilon I$ is added to control the model complexity, and we set the parameter $\epsilon$ as $10^{-2}$ in biometric structure prediction and advertisement classification according to \cite{DP-Foster-et-al-TR-TTI-2008}. The parameter is tuned over the set $\{ 10^i|i=-5,\ldots,4 \}$ in web image annotation. The implementation details can be found in \cite{DP-Foster-et-al-TR-TTI-2008}. For $m$ different views, there are $m(m-1)/2$ subsets of two views. The subset that achieves the best performance is termed \textbf{CCA (BST)}. To combine the results of all subsets, we average their predicted scores in RLS-based classification and adopt the majority voting strategy in $k$NN. This combination approach is termed \textbf{CCA (AVG)}.
  \item \textbf{CCA-LS \cite{J-Via-et-al-NN-2007}:} a generalization of CCA to multiple views based on least square (LS) regression.
  \item \textbf{DSE \cite{B-Long-et-al-SDM-2008}:} a general and popular unsupervised multi-view dimension reduction method based on spectral embedding.
  \item \textbf{SSMVD \cite{YH-Han-et-al-TCSVT-2012}:} a recently proposed unsupervised multi-view dimension reduction method based on the structured sparsity-inducing norm \cite{R-Jenatton-et-al-JMLR-2011}.
  \item \textbf{TCCA:} the proposed tensor CCA. The regularization parameter $\epsilon$ is optimized the same as in CCA.
\end{itemize}
In the first step of DSE and SSMVD, PCA is taken as the dimension reduction method for each view, and the result dimension (of each view) is set to be $100$ empirically.

\subsubsection{Biometric Structure Prediction}

The dataset used in this set of experiments is SecStr\footnote{\url{http://www.kyb.tuebingen.mpg.de/ssl-book}}, which is a benchmark dataset for evaluating semi-supervised systems \cite{O-Chapelle-et-al-Book-MIT-2006}. The task associated with this dataset is ``to predict the secondary structure of a given amino acid in a protein based on a sequence window centered around that amino acid'' \cite{O-Chapelle-et-al-Book-MIT-2006}. The SecStr dataset is large-scale and contains $84K$ instances. We randomly select $100$ instances as labeled samples. There are also $1200K$ unlabeled instances which we use to observe the performance of three CCA-based methods (CCA, CCA-LS and TCCA) with respect to different amounts of unlabeled data. Following \cite{DP-Foster-et-al-TR-TTI-2008}, all the provided data are used (as unlabeled instances) to find the common subspace in the CCA-based methods. The performance is evaluated in a transductive setting on the unlabeled samples (except those for validation) of the $84K$ instances. Both DSE and SSMVD are naturally transductive, since they learn the low-dimensional representation of given data directly, and no projection matrix is learned for new data. Therefore, these two methods cannot handle very large datasets and the experiments are conducted only on the $84K$ instances. In particular, DSE needs to solve an eigen-decomposition problem of size $N \times N$. The time cost or memory cost is intolerable when $N$ is $84K$, and thus a subset of $10K$ samples are utilized.

The features provided are $15$ categorical attributes, each of which is generated at a position in $[-7,+7]$ from the sequence window of amino acid, and represented by a $21$-dimensional sparse binary vector. We divided the $315 (15 \times 21)$ features into three views:
\begin{itemize}
  \item View-1: attributes based on the left context (positions in $[-7,-3]$);
  \item View-2: attributes based on the current position and middle context (positions in $[-2,2]$);
  \item View-3: attributes based on the right context (positions in $[3,7]$).
\end{itemize}
The dimension of each view is $105$.

\begin{figure}[!t]
\centering
  \subfigure{\includegraphics[width=0.45\linewidth]{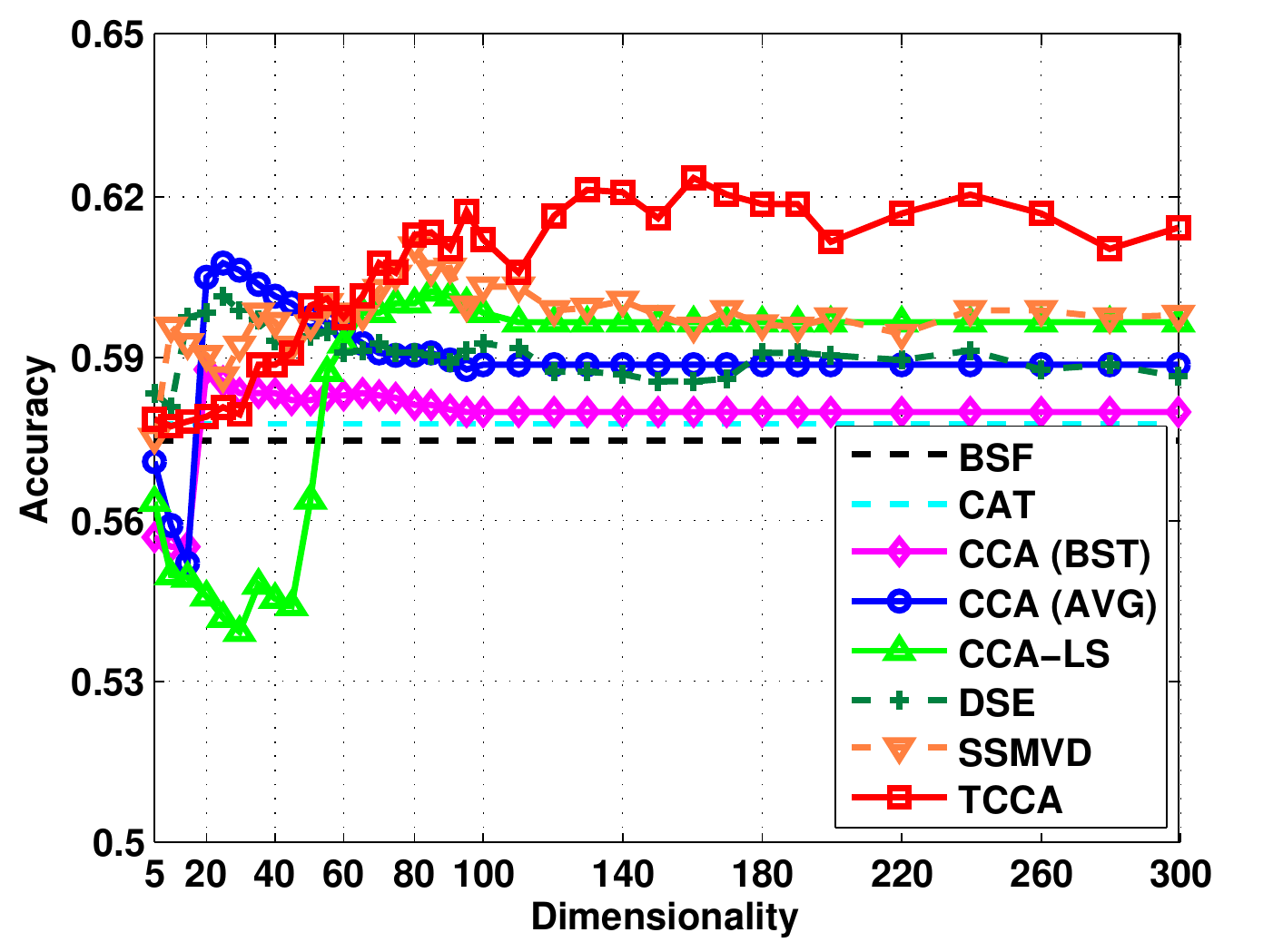}
  \label{subfig:Acc_vs_Dim_Lin_SecStr100_84K}
  }
  \hfil
  \subfigure{\includegraphics[width=0.45\linewidth]{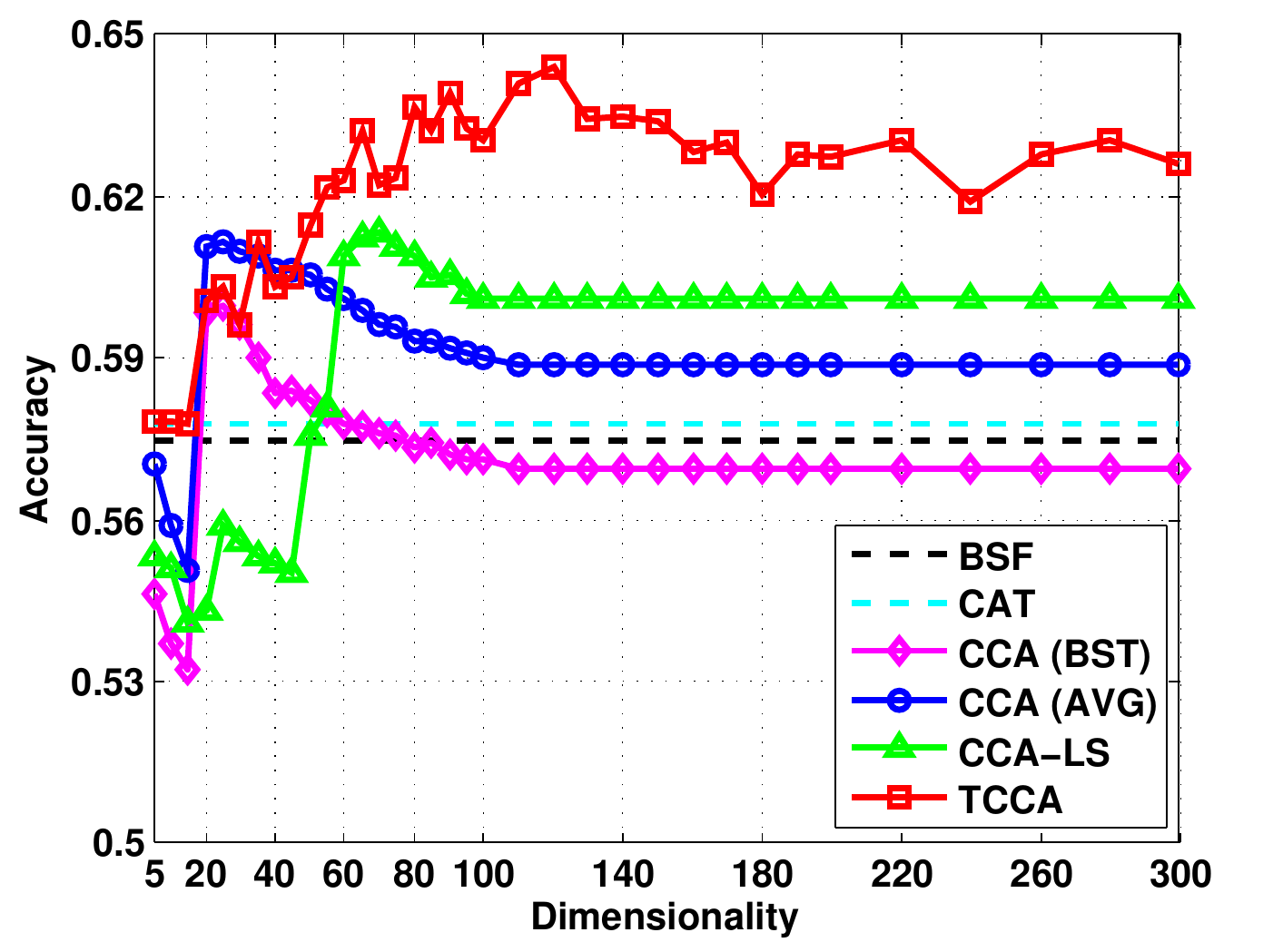}
  \label{subfig:Acc_vs_Dim_Lin_SecStr100_1M}
  }
  \caption{Prediction accuracy vs. dimension of the common subspace on the SecStr dataset. (Top: $100$ labeled instances and $84K$ unlabeled instances; Bottom: $100$ labeled instances and the entire unlabeled set (about $1.3M$ instances).)}
\label{fig:Acc_vs_Dim_Lin_SecStr100}
\end{figure}

\begin{table}[!t]
\renewcommand{\arraystretch}{1.3}
\caption{Prediction accuracies ($\%$) of the different methods at their best dimensions on the SecStr dataset ($100$ labeled instances).}
\label{tab:Acc_vs_Dim_Lin_SecStr100}
\centering
\begin{tabular}{|c||c|c|}
\hline
Methods & \#unlabeled = $84K$ & \#unlabeled = $1.3M$ \\
\hline
BSF & \multicolumn{2}{c|}{57.48$\pm$1.90} \\
\hline
CAT & \multicolumn{2}{c|}{57.77$\pm$2.03} \\
\hline \hline
CCA (BST) & 58.78$\pm$2.97 & 59.97$\pm$2.46 \\
\hline
CCA (AVG) & 60.75$\pm$1.92 & 61.15$\pm$1.73 \\
\hline
CCA-LS & 60.23$\pm$1.70 & 61.32$\pm$1.65 \\
\hline
DSE & 60.15$\pm$0.81 & \multirow{2}{*}{No Attempt} \\
\cline{1-2}
SSMVD & 61.08$\pm$1.58 & \\
\hline
TCCA & \textbf{62.36$\pm$1.27} & \textbf{64.42$\pm$1.70} \\
\hline
\end{tabular}
\end{table}

The performance of the compared methods in relation to the dimension of the common subspace is shown in Fig. \ref{fig:Acc_vs_Dim_Lin_SecStr100}. Accuracy is averaged over $5$ runs for each dimension $r$ in $\{ 5,10, \ldots, 100,110, \ldots ,200,220, \ldots, 300 \}$. The performance of the different methods at their best dimensions are summarized in Table \ref{tab:Acc_vs_Dim_Lin_SecStr100}. From the results, we observe that: 1) the concatenation strategy (CAT) is comparable to and slightly better than the strategy of only using the best single view features (BSF); 2) by learning the common subspace, all the compared multi-view dimension reduction methods are significantly better than the BSF and CAT baselines, if the dimensionalities are properly set according to the accuracy on the validation dataset. In particular, CCA (BST) is superior to CAT, although only a subset of two views is utilized in the former; 3) the accuracy of all three CCA-based methods increases with an increasing number of unlabeled data. By combining the results of different subsets, CCA (AVG) is better than CCA (BST); 4) CCA-LS is superior to CCA (BST), but their performance at their best dimension is comparable. When the number of unlabeled data is $84K$, DSE and SSMVD are comparable to CCA (BST) and CCA-LS respectively; 5) the performance of TCCA does not decease significantly as CCA-LS and CCA do when the number of dimensions is high. The main reason is that the ALS algorithm used in TCCA seeks to maximize the canonical correlations for all the $r$ factors simultaneously, but not to greedily find orthogonal decomposition components \cite{GI-Allen-AISTATS-2012}. That is, the main variance tends to be explained uniformly by all factors, not only by the first several factors. This is also the reason why there are some oscillations in TCCA; 6) the proposed TCCA significantly outperforms all the other methods on most dimensionalities. This demonstrates that the high order correlation information between all features is well discovered, and that exploring this kind of information is much better than only exploring the correlation information between pairs of features, as in CCA-LS.

\subsubsection{Advertisement Classification}

This set of experiments is conducted on the Ads (internet advertisements)\footnote{\url{http://archive.ics.uci.edu/ml/datasets/Internet+Advertisements}} dataset from the well-known UCI Machine Learning Repository. The task is to predict whether or not a given hyperlink (associated with an image) is an advertisement. There are $3,279$ instances in this dataset. We randomly choose $100$ instances as labeled training samples, and all the instances except those for validation are utilized as unlabeled samples to find the common subspace. The performance is evaluated in a transductive setting on the unlabeled samples.

We use the features as described in \cite{N-Kushmerick-ICAA-1999}, and omit the attributes that have missing values, such as the height (and width) of the image. The remained attributes are represented by binary ($1/0$) features which indicate the presence/absence of corresponding terms. For CCA-LS and TCCA, we divide all these features into three views as follows:
\begin{itemize}
  \item View-1: features based on the terms in the image¡¯s URL, caption, and alt text. $588$ dimensions;
  \item View-2: features based on the terms in the URL of the current site. $495$ dimensions;
  \item View-3: features based on the terms in the anchor URL. $472$ dimensions.
\end{itemize}

\begin{figure}[!t]
\centering
\includegraphics[width=0.5\linewidth]{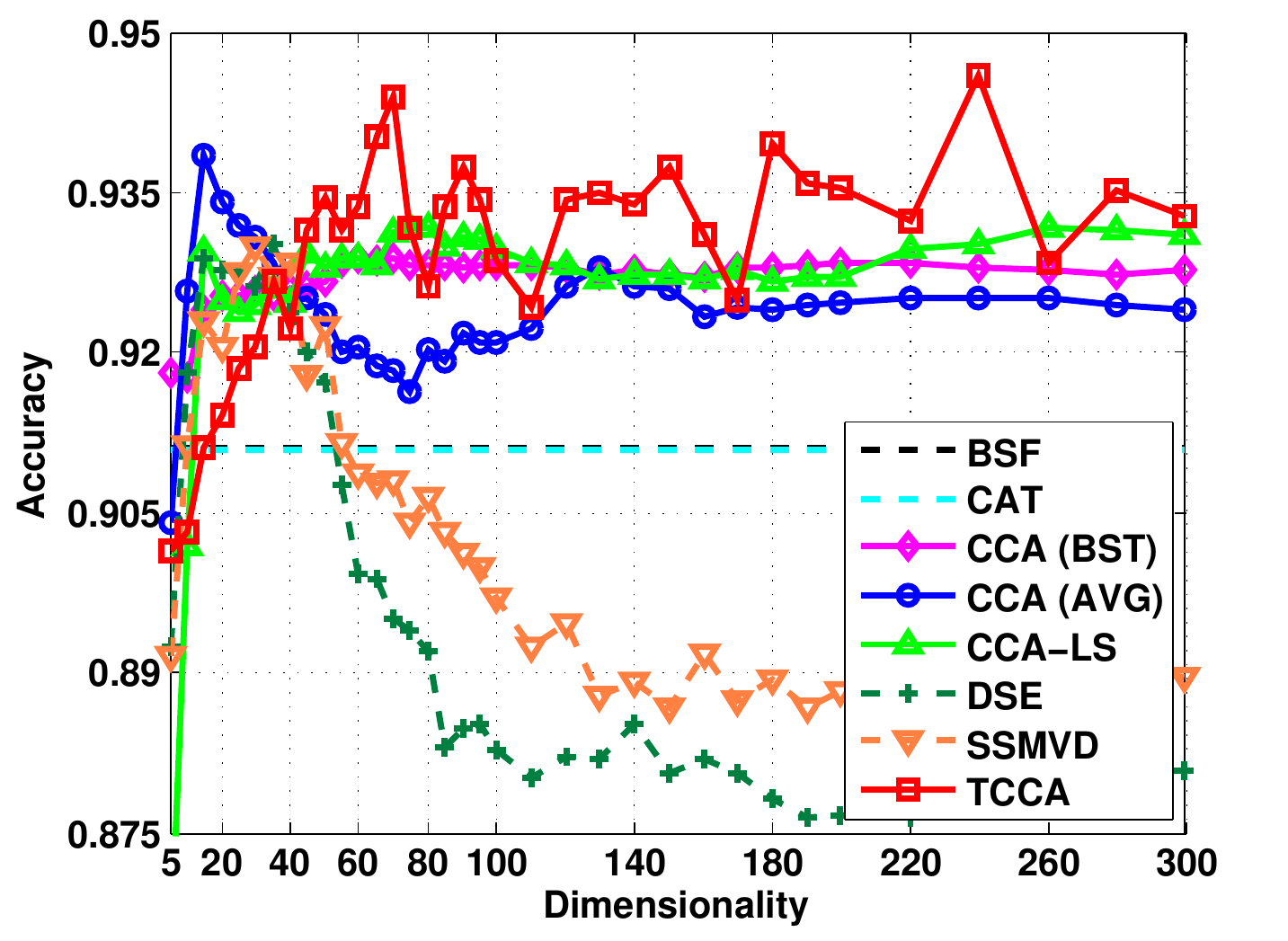}
\caption{Classification accuracy vs. dimension of the common subspace on the Ads dataset. $100$ labeled training samples are utilized.}
\label{fig:Acc_vs_Dim_Lin_Ads}
\end{figure}

\begin{table}[!t]
\renewcommand{\arraystretch}{1.3}
\caption{Classification accuracies ($\%$) of the different methods at their best dimensions on the Ads dataset.}
\label{tab:Acc_vs_Dim_Lin_Ads}
\centering
\begin{tabular}{|c||c|}
\hline
Methods & \#labeled = 100 \\
\hline
BSF & 91.10$\pm$1.65 \\
\hline
CAT & 91.08$\pm$1.74 \\
\hline \hline
CCA (BST) & 92.88$\pm$1.11 \\
\hline
CCA (AVG) & 93.84$\pm$0.85 \\
\hline
CCA-LS & 93.17$\pm$1.10 \\
\hline
DSE & 93.01$\pm$0.96 \\
\cline{1-2}
SSMVD & 92.99$\pm$0.91 \\
\hline
TCCA & \textbf{94.59$\pm$0.27} \\
\hline
\end{tabular}
\end{table}

Fig. \ref{fig:Acc_vs_Dim_Lin_Ads} shows the classification accuracy of the compared methods (in relation to the dimension $r$), and the accuracies at their best dimensions are summarized in Table \ref{tab:Acc_vs_Dim_Lin_Ads}. In contrast to the observations of the last set of experiments, we can see that: 1) the accuracy of the concatenation strategy (CAT) and the best single view (BSF) are almost the same. The performance of CAT is relatively worse since the feature dimension in this set of experiments is high ($1,555$ dimensions), and over-fitting occurs given the limited number of labeled samples; 2) the performance of DSE and SSMVD first increase and then decrease sharply with an increasing number of the dimension $r$, while the CCA-based methods are much steady; 3) the improvement of TCCA compared with the other CCA-based methods is not as great as in the last set of experiments. This is because we need more samples to approximate the true underlying high order correlation compared with the traditional pairwise correlation, since there are more variables to be estimated in the high order statistics. The unlabeled instances utilized in this set of experiments are much fewer, thus the high order correlation information is not well explored. CCA-LS is only comparable to CCA for the same reason.

\subsubsection{Web Image Annotation}

We further verify the effectiveness of the proposed algorithm on a natural image dataset NUS-WIDE \cite{TS-Chua-et-al-CIVR-2009}. This dataset contains $269,648$ images, and our experiments are conduct on a subset that consists of $11,189$ images belonging to $10$ mammal concepts: bear, cat, cow, dog, elk, fox, horse, tiger, whale, and zebra. We randomly split the images into a training set of $5,597$ images and a test set of $5,592$ images. Distinguishing between these concepts is very challenging, since many of them are similar to each other, e.g., cat and tiger. We randomly choose $\{ 4,6,8 \}$ labeled instances for each concept in the training set, and all the training instances are utilized as unlabeled samples to find the common subspace.

In this dataset, we choose three types of visual feature, namely $500$-D bag of visual words based on SIFT \cite{DG-Lowe-IJCV-2004} descriptors, $144$-D color auto-correlogram, and $128$-D wavelet texture, to represent each image \cite{TS-Chua-et-al-CIVR-2009}.

The annotation performance of the compared methods is shown in Fig. \ref{fig:Acc_vs_Dim_Lin_NUS} and Table \ref{tab:Acc_vs_Dim_Lin_NUS}. It can be seen from the results that: 1) in general, performance improves with an increased number of labeled instances; 2) CCA-LS is comparable to CCA (BST) and CCA (AVG), while the best performance (peak of the curve) of CCA-LS is usually higher; 3) the performance of DSE is poor when $r$ is large, while SSMVD is much steady and can be superior to CCA (AVG) and CCA-LS sometimes; 4) the accuracies of CCA (AVG) and CCA-LS first increase and then decrease with an increasing number of the dimension $r$, while the results of the proposed TCCA are satisfactory even though $r$ is large; 3) the accuracy of TCCA is significantly better than that of all the other methods under most dimensionalities.

\begin{figure*}[!t]
\centering
  \subfigure{\includegraphics[width=0.32\linewidth]{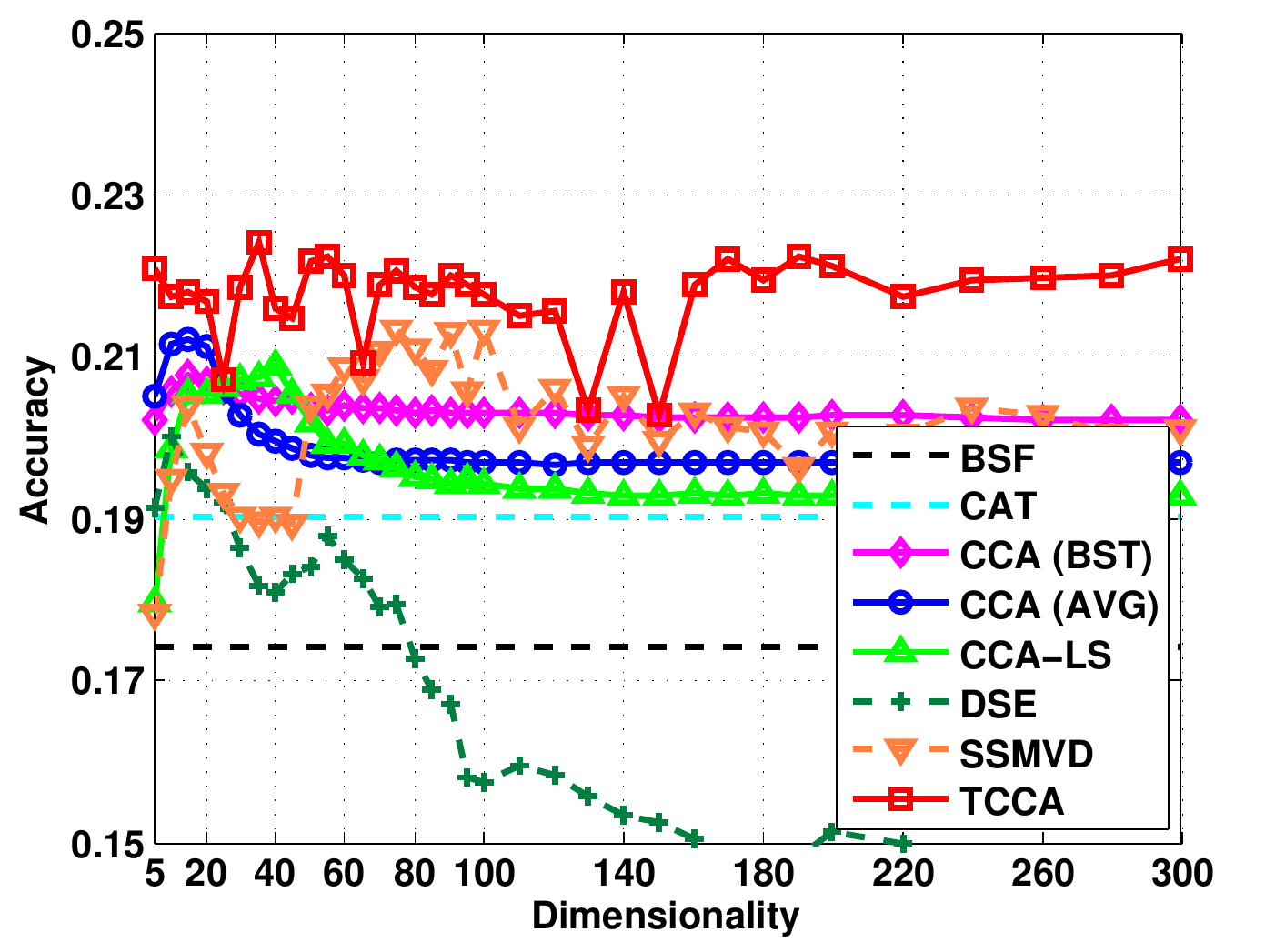}
  \label{subfig:Acc_vs_Dim_Lin_NUS4}
  }
  \hfil
  \subfigure{\includegraphics[width=0.32\linewidth]{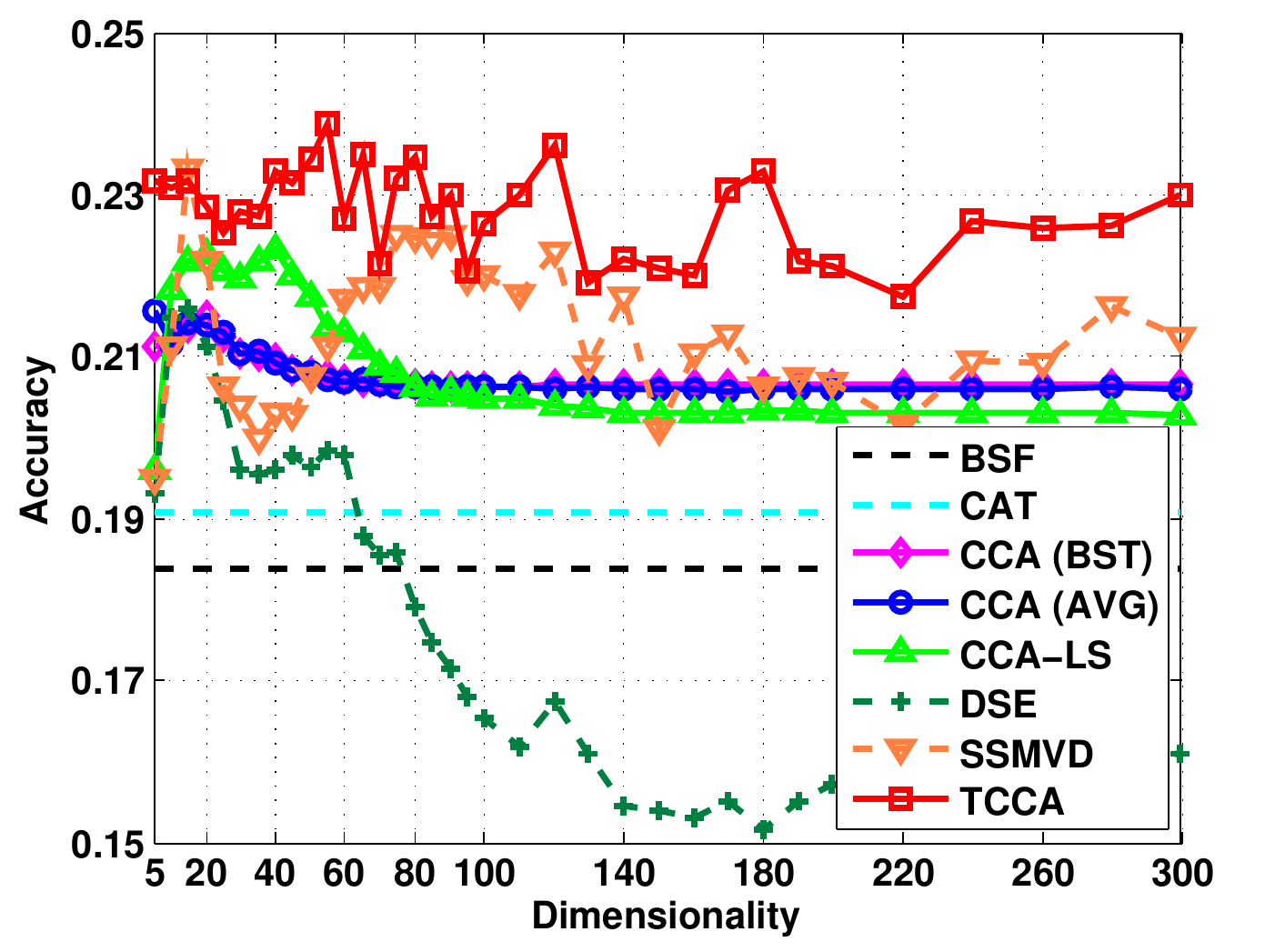}
  \label{subfig:Acc_vs_Dim_Lin_NUS6}
  }
  \hfil
  \subfigure{\includegraphics[width=0.32\linewidth]{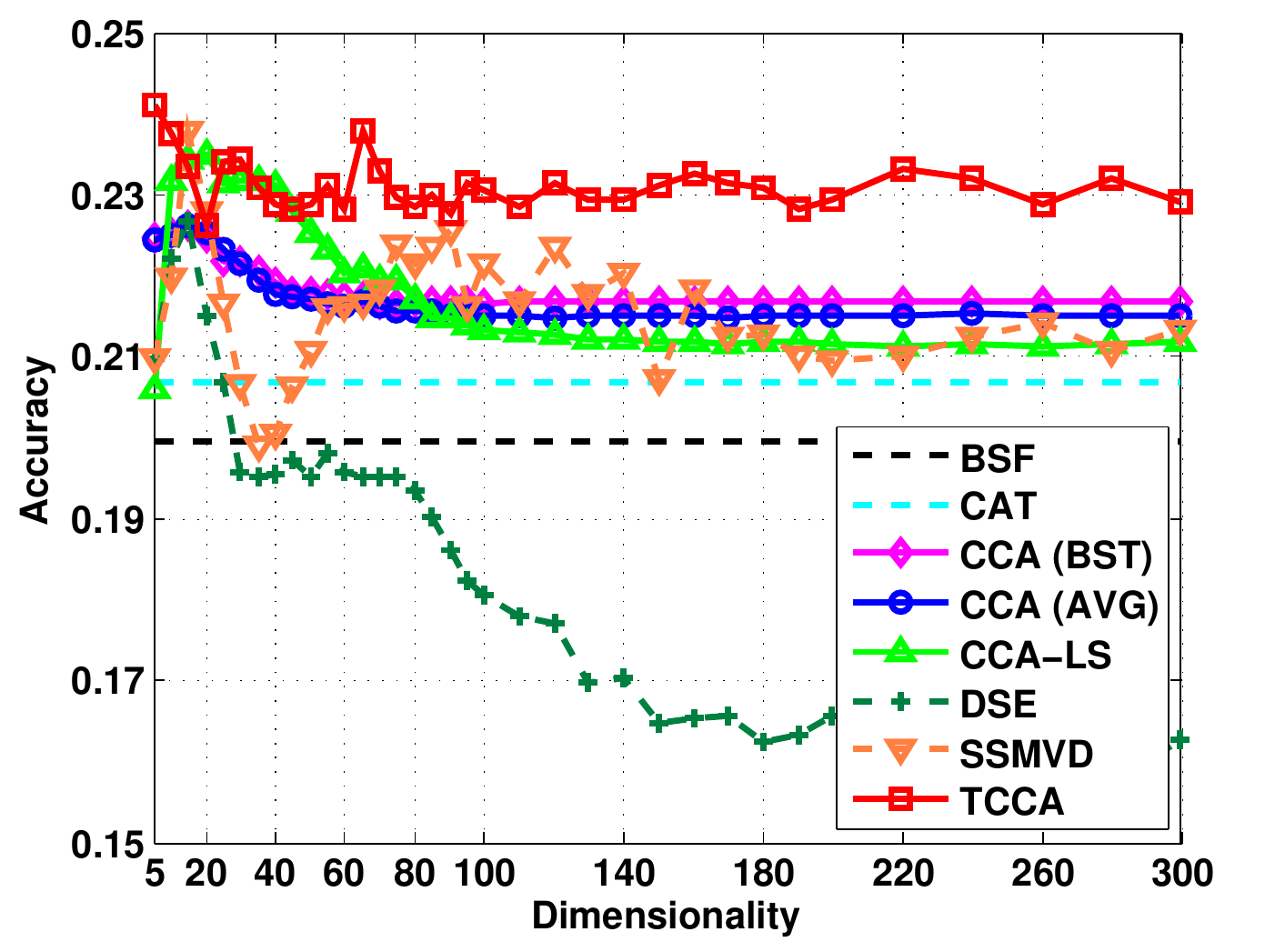}
  \label{subfig:Acc_vs_Dim_Lin_NUS8}
  }
  \caption{Anotation accuracy vs. dimension of the common subspace on the NUS-WIDE mammal subset. (Left: 4 labeled instances for each mammal concept; Middle: 6 labeled instances; Right: 8 labeled instances.)}
\label{fig:Acc_vs_Dim_Lin_NUS}
\end{figure*}

\begin{table}[!t]
\renewcommand{\arraystretch}{1.3}
\caption{Annotation accuracies ($\%$) of the different methods at their best dimensions on the NUS-WIDE mammal dataset.}
\label{tab:Acc_vs_Dim_Lin_NUS}
\centering
\begin{tabular}{|c||c|c|c|}
\hline
Methods & \#labeled = $4$ & \#labeled = $6$ & \#labeled = $8$ \\
\hline
BSF & 17.42$\pm$1.37 & 18.37$\pm$1.21 & 19.96$\pm$1.19 \\
\hline
CAT & 19.01$\pm$1.86 & 19.07$\pm$2.23 & 20.70$\pm$1.44 \\
\hline \hline
CCA (BST) & 20.77$\pm$1.52 & 21.51$\pm$2.38 & 22.61$\pm$1.76 \\
\hline
CCA (AVG) & 21.21$\pm$1.47 & 21.57$\pm$2.04 & 22.61$\pm$1.21 \\
\hline
CCA-LS & 20.90$\pm$1.84 & 22.31$\pm$2.53 & 23.50$\pm$2.48 \\
\hline
DSE & 20.02$\pm$1.23 & 21.59$\pm$1.13 & 22.67$\pm$0.74 \\
\hline
SSMVD & 21.34$\pm$2.08 & 23.32$\pm$1.08 & 23.79$\pm$1.30 \\
\hline
TCCA & \textbf{22.40$\pm$1.96} & \textbf{23.86$\pm$1.41} & \textbf{24.11$\pm$0.32} \\
\hline
\end{tabular}
\end{table}

\subsection{Evaluation of the Non-linear Extension}

We evaluate the non-linear extension of the proposed TCCA in the web image annotation task. As discussed in Section \ref{subsec:Complexity_Analysis}, the non-linear extension is able to handle the small sample size problem, where the feature dimensions can be very high and possibly infinite. We thus randomly choose a small set of $500$ samples from the animal subset. To perform the non-linear classification, we construct a kernel for each kind of feature. The kernel is defined by
\begin{equation}\notag
k(\mathbf{x}_i, \mathbf{x}_j) = \mathrm{exp}(-\lambda^{-1} d(\mathbf{x}_i, \mathbf{x}_j)),
\end{equation}
where $d(\mathbf{x}_i, \mathbf{x}_j)$ denotes the distance between $\mathbf{x}_i$ and $\mathbf{x}_j$, and $\lambda = \mathrm{max}_{i,j} d(\mathbf{x}_i, \mathbf{x}_j)$. We choose the $\chi^2$ distance for the visual word histogram. For other features, the $L2$ distance is utilized. Specifically, we compare the following methods:
\begin{itemize}
  \item \textbf{BSK:} using the single view kernel that achieves the best performance in the $k$NN-based classification.
  \item \textbf{AVG:} averaging the normalized kernels of all the views, and then performing $k$NN-based classification.
  \item \textbf{KCCA \cite{DR-Hardoon-et-al-NCn-2004}:} using the KCCA formulation presented in \cite{DR-Hardoon-et-al-NCn-2004} to find a common representation of two different views. The regularization parameter is optimized over the set $\{ 10^i| i=-7, \ldots,2 \}$. The setup of \textbf{KCCA (BST)} and \textbf{KCCA (AVG)} are similar as \textbf{CCA (BST)} and \textbf{CCA (AVG)} in the experiments of the linear version.
  \item \textbf{KTCCA:} the non-linear extension of the proposed tensor CCA. The regularization parameter $\epsilon$ is optimized in the same way as in KCCA.
\end{itemize}

The experimental results are shown in Fig. \ref{fig:Acc_vs_Dim_Ker_NUS} and Table \ref{tab:Acc_vs_Dim_Ker_NUS}. Compared with the results in Fig. \ref{fig:Acc_vs_Dim_Lin_NUS}, we can see that: 1) although a small number of unlabeled samples is utilized, the performance is better since the separability is improved by the non-linear projection, which is implemented via the kernel trick \cite{J-Shawe-Taylor-and-N-Cristianini-Book-Cambridge-2004}; 2) the simple AVG view combination strategy outperforms the best single view kernel (BSK) significantly, and is comparable to KCCA (BST); 3) KCCA (AVG) is slightly better than KCCA (BST), and the proposed KTCCA achieves the best performance under most dimensionalities.

\begin{figure*}[!t]
\centering
  \subfigure{\includegraphics[width=0.32\linewidth]{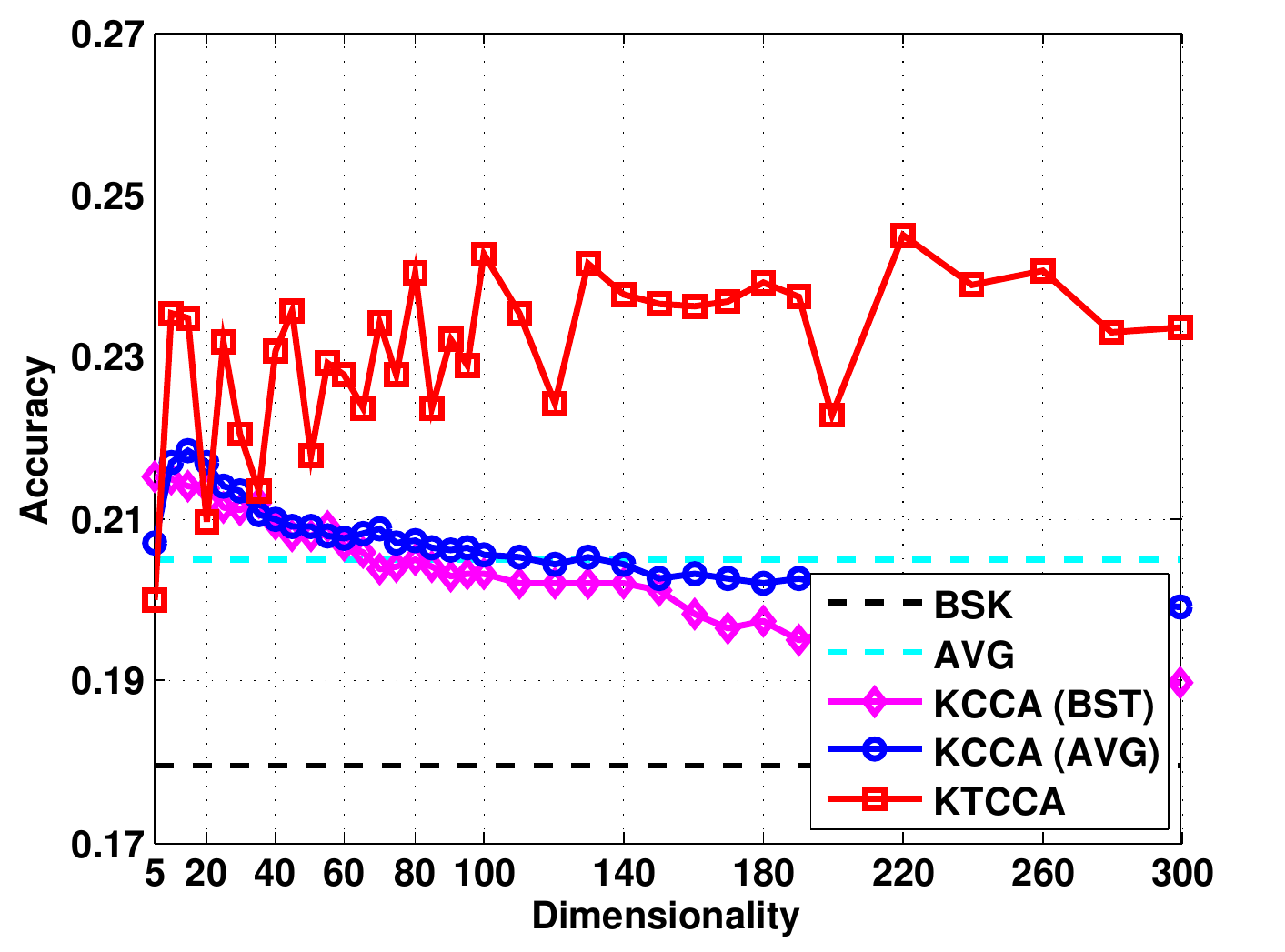}
  \label{subfig:Acc_vs_Dim_Ker_NUS4}
  }
  \hfil
  \subfigure{\includegraphics[width=0.32\linewidth]{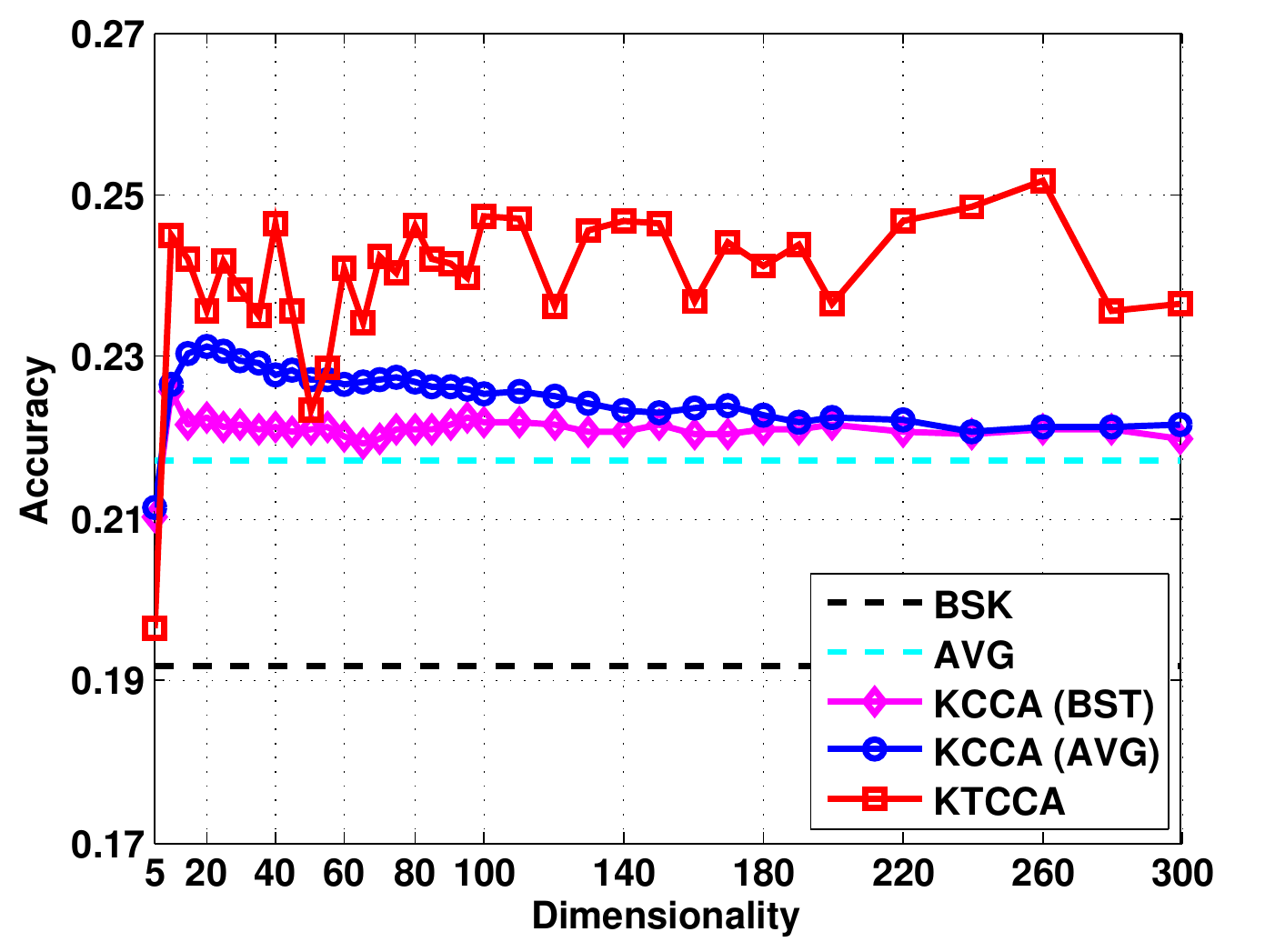}
  \label{subfig:Acc_vs_Dim_Ker_NUS6}
  }
  \hfil
  \subfigure{\includegraphics[width=0.32\linewidth]{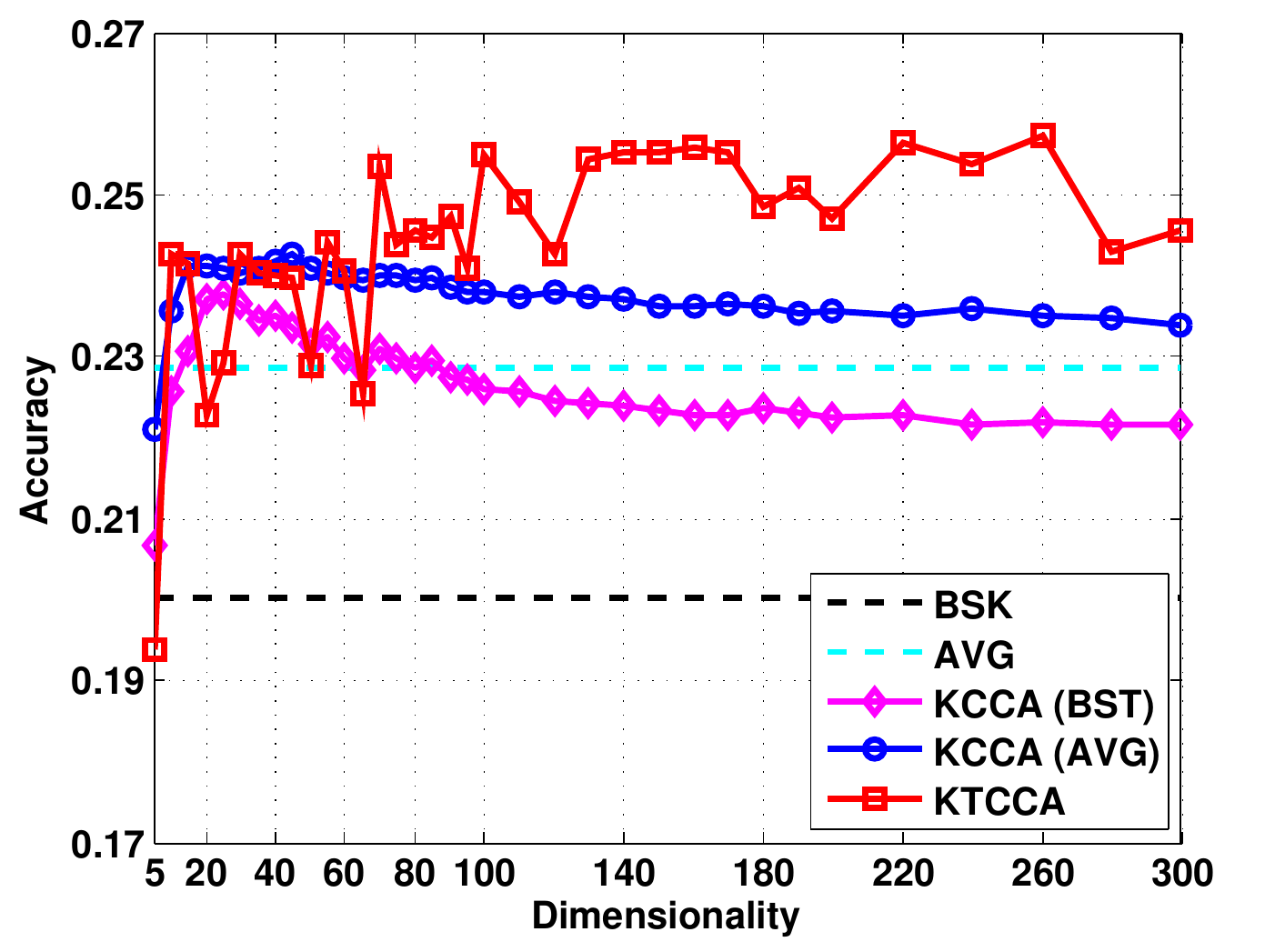}
  \label{subfig:Acc_vs_Dim_Ker_NUS8}
  }
  \caption{Annotation accuracy (of the non-linear methods) vs. dimension of the common subspace on the NUS-WIDE mammal subset, where a small set of $500$ samples is utilized. (Left: 4 labeled instances for each mammal concept; Middle: 6 labeled instances; Right: 8 labeled instances.)}
\label{fig:Acc_vs_Dim_Ker_NUS}
\end{figure*}

\begin{table}[!t]
\renewcommand{\arraystretch}{1.3}
\caption{Annotation accuracies ($\%$) of the different non-linear methods at their best dimensions on the NUS-WIDE mammal dataset.}
\label{tab:Acc_vs_Dim_Ker_NUS}
\centering
\begin{tabular}{|c||c|c|c|}
\hline
Methods & \#labeled = $4$ & \#labeled = $6$ & \#labeled = $8$ \\
\hline
BSK & 17.96$\pm$1.29 & 19.17$\pm$2.01 & 20.04$\pm$1.66 \\
\hline
AVG & 20.49$\pm$1.65 & 21.73$\pm$2.74 & 22.86$\pm$1.87 \\
\hline \hline
KCCA (BST) & 21.51$\pm$2.44 & 22.58$\pm$1.91 & 23.78$\pm$1.57 \\
\hline
KCCA (AVG) & 21.85$\pm$1.38 & 23.13$\pm$1.77 & 24.28$\pm$1.04 \\
\hline
KTCCA & \textbf{24.51$\pm$0.78} & \textbf{25.18$\pm$0.58} & \textbf{25.74$\pm$0.90} \\
\hline
\end{tabular}
\end{table}

\subsection{Empirical analysis of the computational complexity}

In this subsection, we empirically analyze the computational complexity of the different methods. The experiments are conducted in Matlab R2012b on a $2 \times 3.33$ GHz Intel Xeon ($6$ cores) computer, where the memory is $48$GB $1333$MHz ECC DDR3-RAM. The results (time cost and memory cost) on the different datasets are shown in Fig. \ref{fig:Cst_vs_Dim_Lin_SecStr}-\ref{fig:Cst_vs_Dim_Ker_NUS}. From the results, we observe that: 1) the costs of the proposed TCCA are higher than the other CCA-based methods in general. This is because the decomposition is performed on a large $d_1 \times d_2 \times \ldots \times d_m$ covariance tensor, instead of one or multiple $d_p \times d_q$ covariance matrices, where $p,q = 1, \ldots ,m$ are the view indices. The tensor decomposition method we adopt in this paper is the ALS algorithm \cite{PM-Kroonenberg-and-J-De-Leeuw-Psychometrika-1980, P-Comon-et-al-JoC-2009}, which could result in satisfactory accuracy but is not efficient; 2) TCCA is much more efficient than DSE or SSMVD when the feature dimensions are not very high and the number of instances is large (see Fig. \ref{fig:Cst_vs_Dim_Lin_SecStr} for example). This demonstrates the superiority of TCCA compared with the existed unsupervised multi-view dimension reduction methods on the large sample size problems.

\section{Conclusion}\label{sec:Conclusion}
Standard CCA cannot deal with multi-view data, and its typical multi-view extensions ignore the high order statistics (correlation information) among all feature views. To resolve this problem, we have presented tensor CCA (TCCA) to discover such statistics by analyzing the covariance tensor of all views.

From the experimental validation on a variety of application tasks, we conclude that: 1) finding a common subspace for all views using the CCA-based strategy is often better than simply concatenating all the features, especially when the feature dimension is high; 2) examining more statistics, which may require more unlabeled data to be utilized, often leads to better performance; 3) by exploring the high order statistics, the proposed TCCA outperforms the other methods, especially when the dimension of the common subspace is high.

Compared with CCA and its traditional multi-view extensions, the main disadvantage of the proposed TCCA is the high computational cost. Most of the TCCA cost lies in the tensor decomposition, which is not the point of this paper. In the future, we will devote efficient tensor decomposition methods that could speed up TCCA, or introduce the parallel computing technique by utilizing GPU to accelerate the ALS tensor decomposition.

\begin{figure}[!t]
\centering
  \subfigure{\includegraphics[width=0.45\linewidth]{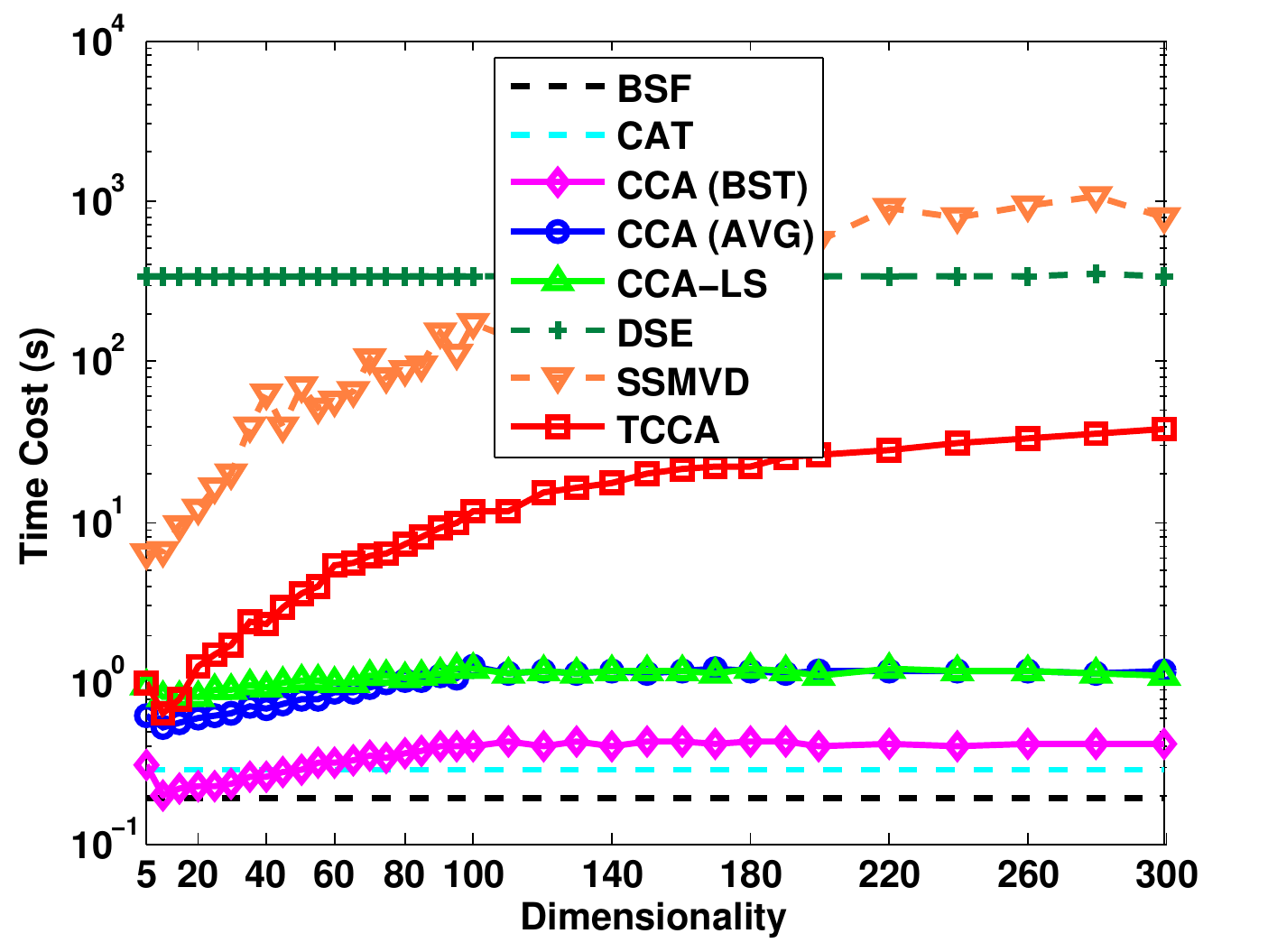}
  \label{subfig:TCst_vs_Dim_Lin_SecStr}
  }
  \hfil
  \subfigure{\includegraphics[width=0.45\linewidth]{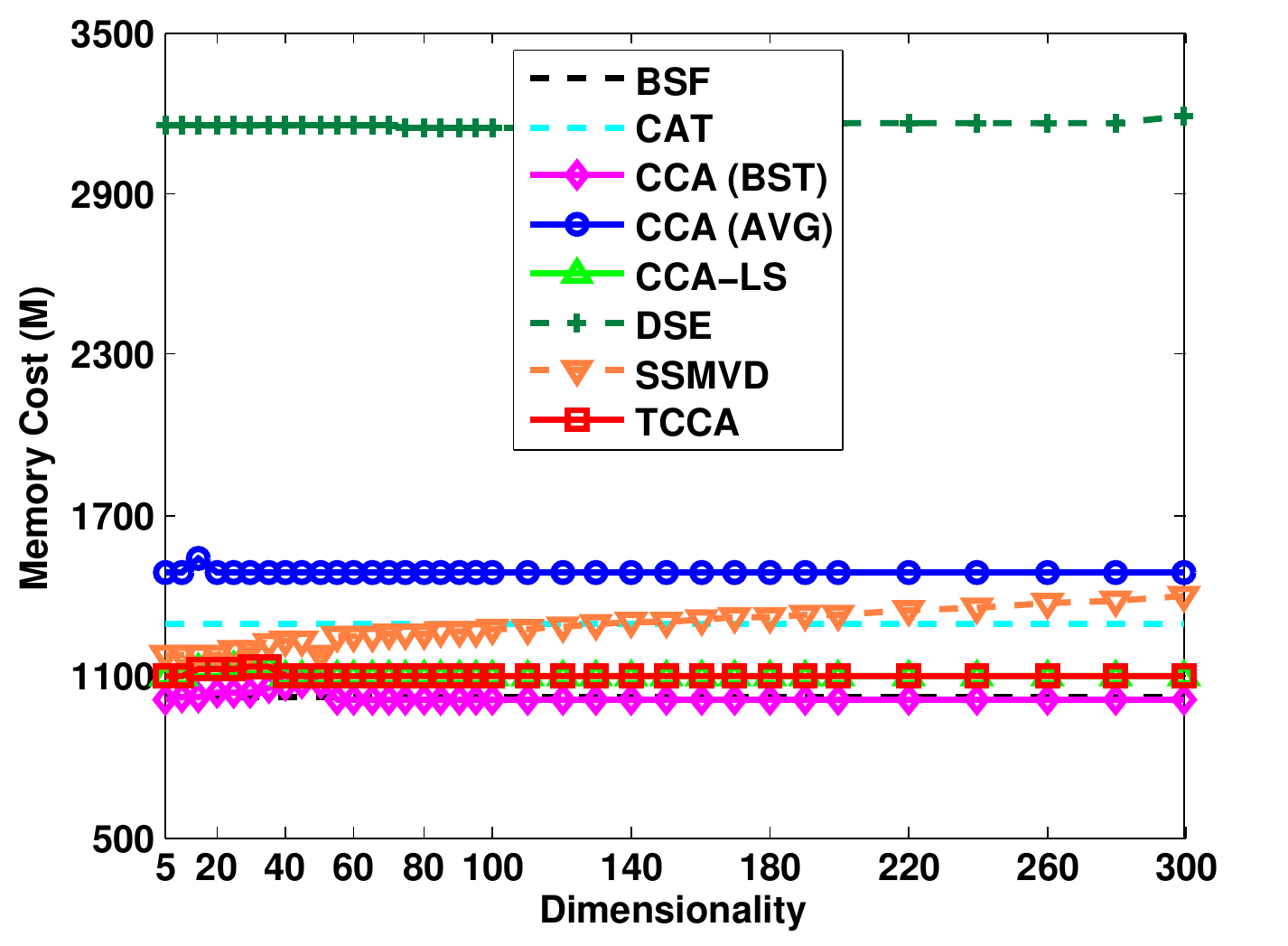}
  \label{subfig:MCst_vs_Dim_Lin_SecStr}
  }
  \caption{Computational complexity vs. dimension of the common subspace on the SecStr dataset. (Top: time cost in seconds; Bottom: memory cost in Megabits.)}
\label{fig:Cst_vs_Dim_Lin_SecStr}
\end{figure}

\begin{figure}[!t]
\centering
  \subfigure{\includegraphics[width=0.45\linewidth]{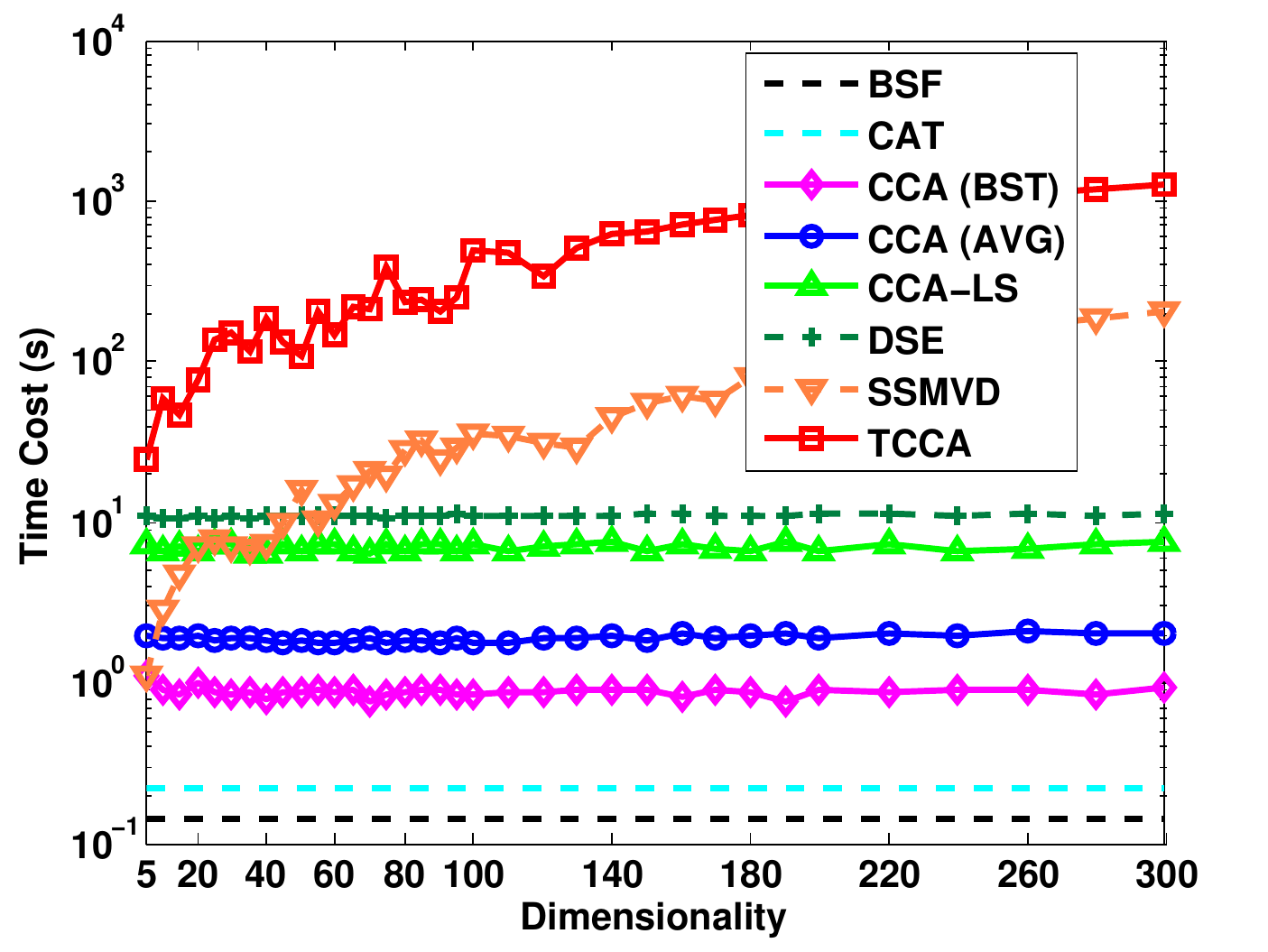}
  \label{subfig:TCst_vs_Dim_Lin_Ads}
  }
  \hfil
  \subfigure{\includegraphics[width=0.45\linewidth]{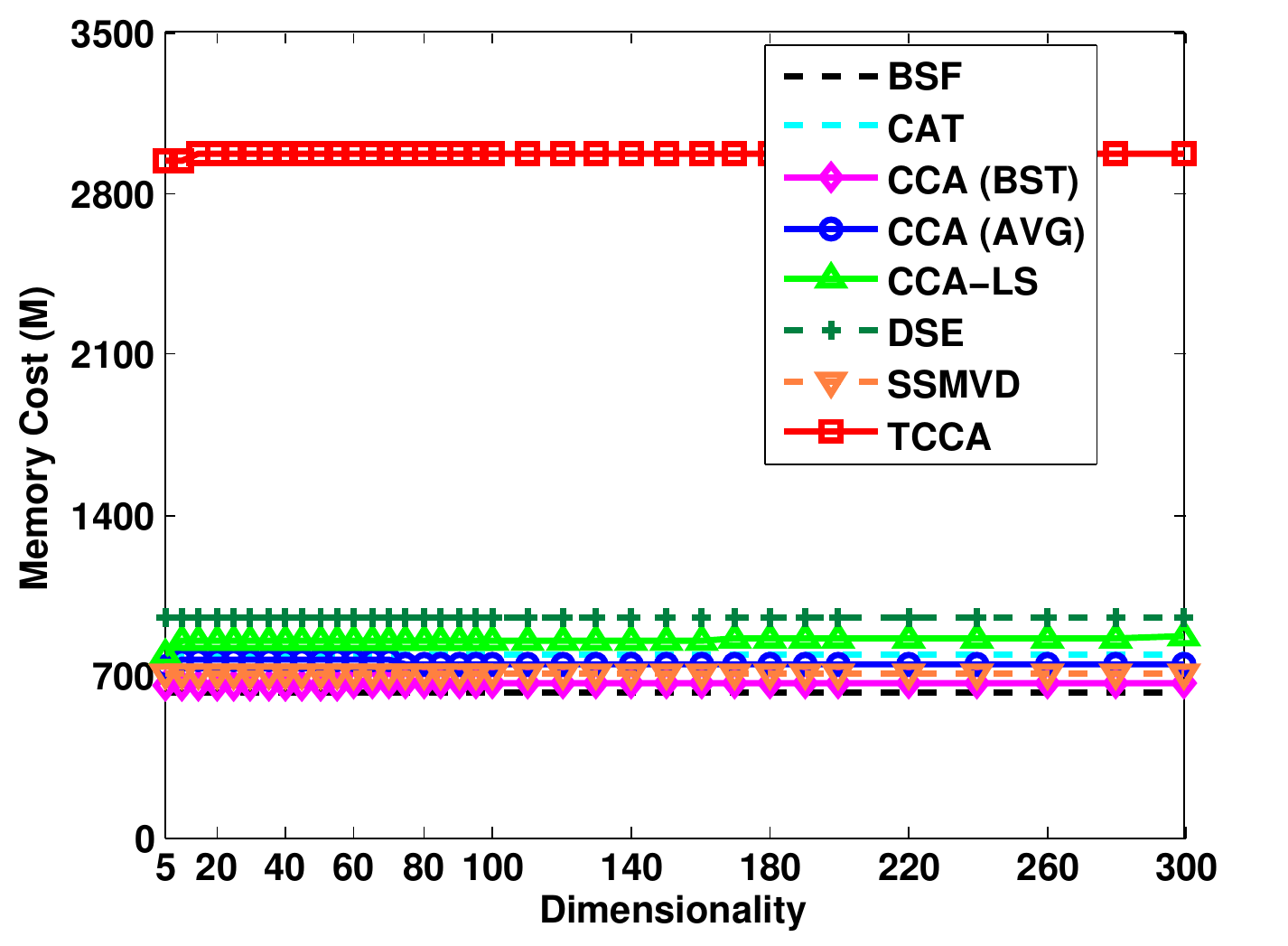}
  \label{subfig:MCst_vs_Dim_Lin_Ads}
  }
  \caption{Computational complexity vs. dimension of the common subspace on the Ads dataset. $100$ labeled training samples are utilized. (Top: time cost in seconds; Bottom: memory cost in Megabits.)}
\label{fig:Cst_vs_Dim_Lin_Ads}
\end{figure}

\begin{figure}[!t]
\centering
  \subfigure{\includegraphics[width=0.45\linewidth]{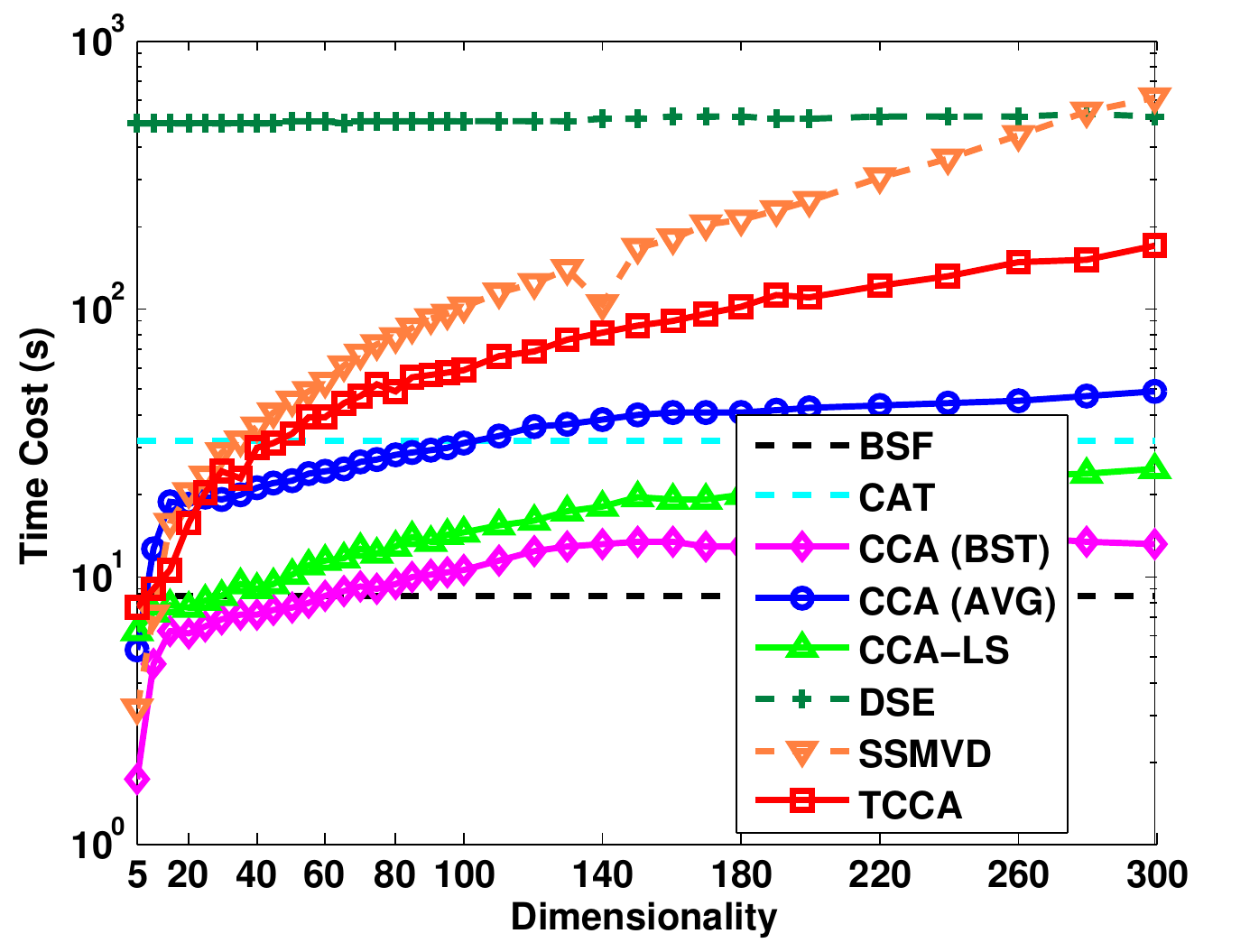}
  \label{subfig:TCst_vs_Dim_Lin_NUS}
  }
  \hfil
  \subfigure{\includegraphics[width=0.45\linewidth]{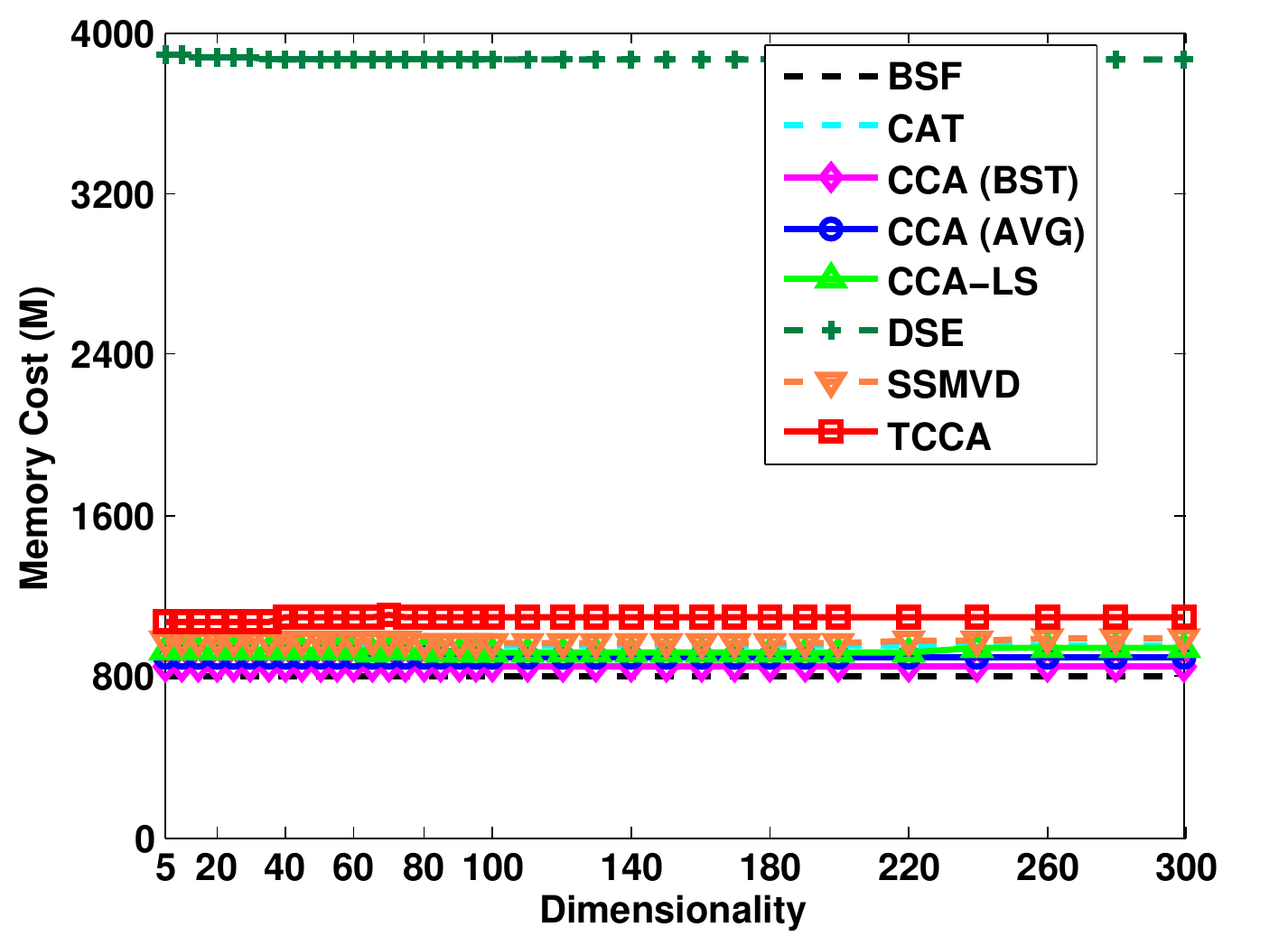}
  \label{subfig:MCst_vs_Dim_Lin_NUS}
  }
  \caption{Computational complexity vs. dimension of the common subspace on the NUS-WIDE mammal subset. $6$ labeled samples for each mammal concept are utilized. (Top: time cost in seconds; Bottom: memory cost in Megabits.)}
\label{fig:Cst_vs_Dim_Lin_NUS}
\end{figure}

\begin{figure}[!t]
\centering
  \subfigure{\includegraphics[width=0.45\linewidth]{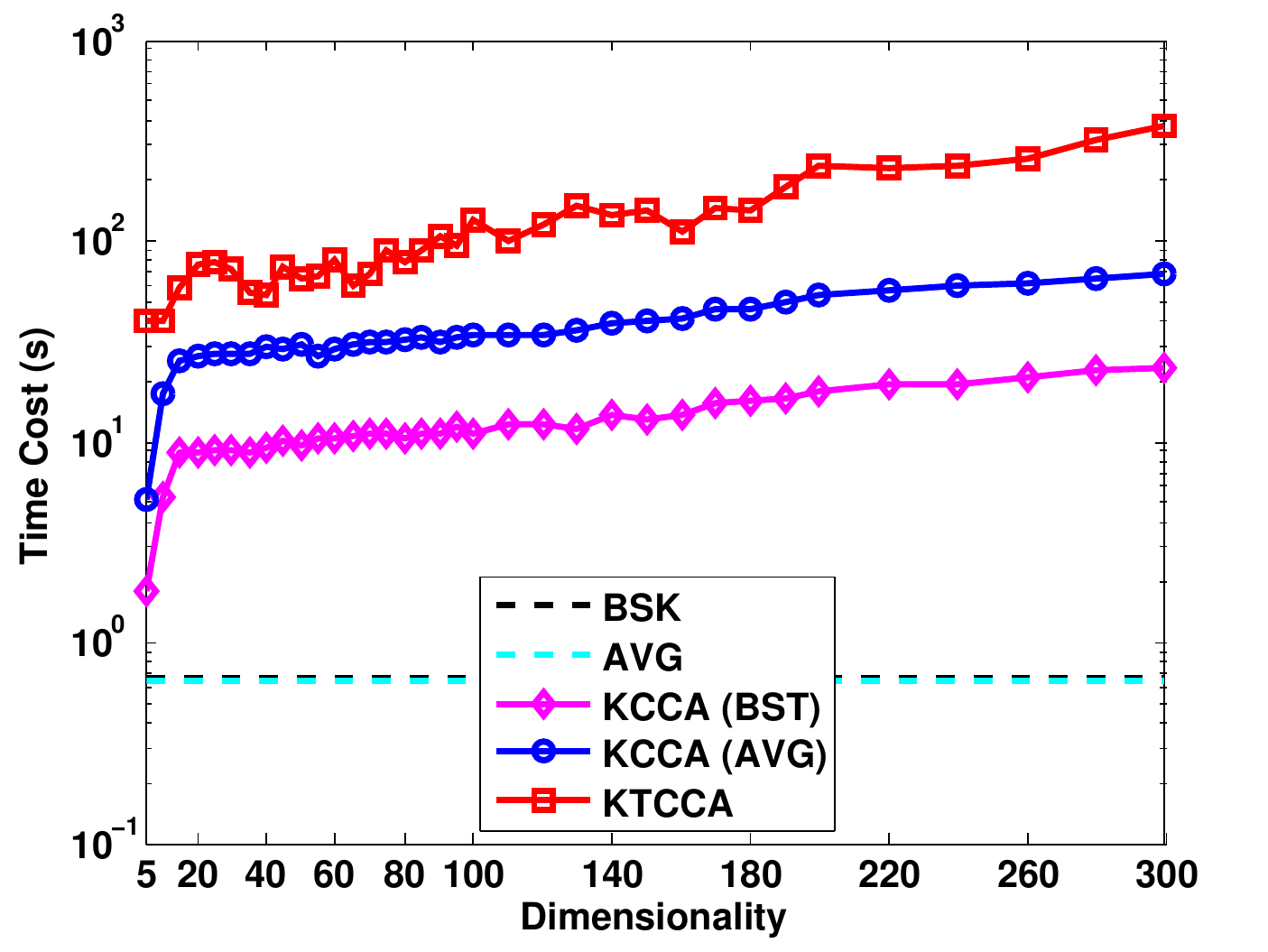}
  \label{subfig:TCst_vs_Dim_Ker_NUS}
  }
  \hfil
  \subfigure{\includegraphics[width=0.45\linewidth]{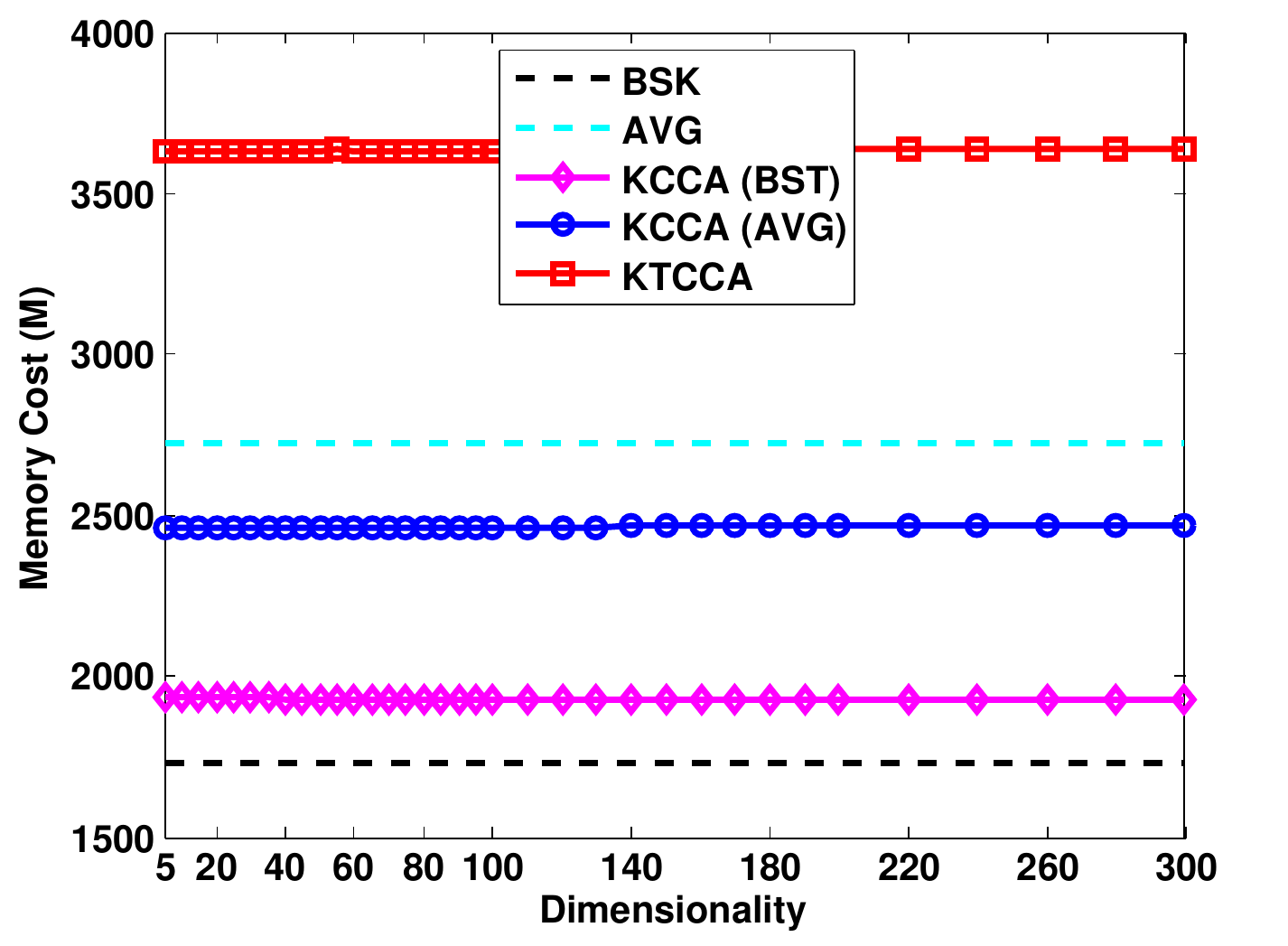}
  \label{subfig:MCst_vs_Dim_Ker_NUS}
  }
  \caption{Computational complexity (of the non-linear methods) vs. dimension of the common subspace on the NUS-WIDE mammal subset, where a small set of $500$ instances and $6$ labeled samples for each mammal concept are utilized. (Top: time cost in seconds; Bottom: memory cost in Megabits.)}
\label{fig:Cst_vs_Dim_Ker_NUS}
\end{figure}

\appendix

\section{Proof of Thoerem \ref{thm:Canonical_Correlation}}
\begin{proof}
\label{prf:Canonical_Correlation}
According to the definition of the element-wise product, we have
\begin{equation}
\begin{split}
\rho = & (\mathbf{z}_1 \odot \mathbf{z}_2 \odot \ldots \odot \mathbf{z}_m)^T \mathbf{e} = \sum_{n=1}^N \mathbf{z}_1(n) \mathbf{z}_2(n) \ldots \mathbf{z}_m(n)
= \sum_{n=1}^N \prod_{p=1}^m \mathbf{z}_p(n) = \sum_{n=1}^N \prod_{p=1}^m \left( \sum_{j_p=1}^{d_p} \mathbf{x}_{pn}(j_p) \mathbf{h}(j_p) \right),
\end{split}
\end{equation}
where $\mathbf{z}_p(n)$ denotes the $n$'th entry of the vector $\mathbf{z}_p$, and the same notation is used for $\mathbf{x}_{pn}$ and $\mathbf{h}$. Additionally,
\begin{equation}
\begin{split}
\mathcal{C}_{12 \ldots m} (j_1, j_2, \ldots, j_m) = & \sum_{n=1}^N \mathbf{x}_{1n}(j_1) \mathbf{x}_{2n}(j_2) \ldots \mathbf{x}_{mn}(j_m) = \sum_{n=1}^N \prod_{p=1}^m \mathbf{x}_{pn}(j_p).
\end{split}
\end{equation}
According to the definition of the $p$-mode product of a tensor and matrix, we have
\begin{equation}
\begin{split}
& (\mathcal{C} \times_p \mathbf{h}_p^T )(j_1, \ldots ,j_{p-1},1,j_{p+1}, \ldots ,j_m) \\
= & \sum_{j_p=1}^{d_p} \mathcal{C}(j_1,j_2, \ldots, j_m) \mathbf{h}(j_p) = \sum_{j_p=1}^{d_p} \left( \sum_{n=1}^N \prod_{p=1}^m \mathbf{x}_{pn}(j_p) \right) \mathbf{h}(j_p)
= \sum_{n=1}^N \sum_{j_p=1}^{d_p} \left( \prod_{p=1}^m \mathbf{x}_{pn}(j_p) \right) \mathbf{h}(j_p).
\end{split}
\end{equation}
Therefore,
\begin{equation}
\begin{split}
& (\mathcal{C} \times_1 \mathbf{h}_1^T \times_2 \mathbf{h}_2^T \ldots \times_m \mathbf{h}_m^T)(1, \ldots ,1,1,1, \ldots ,1)
= \sum_{n=1}^N \prod_{p=1}^m \left( \sum_{j_p=1}^{d_p} \mathbf{x}_{pn}(j_p) \mathbf{h}(j_p) \right).
\end{split}
\end{equation}
This completes the proof.
\end{proof}

\section{Proof of Theorem \ref{thm:Tensor_Kernel_Calculation}}
\begin{proof}
\label{prf:Tensor_Kernel_Calculation}
Let $\mathcal{F} = \mathcal{C}_{12 \ldots m} \times_1 \phi^T(X_1) \times_2 \phi^T(X_2) \ldots \times_m \phi^T(X_m)$ and $\mathcal{G} = \frac{1}{N} \sum_{n=1}^N \mathbf{k}_{1n} \circ \mathbf{k}_{2n} \circ \ldots \circ \mathbf{k}_{mn}$, then according to the definition of the outer product, the $(j_1,j_2, \ldots, j_m)$'th entry of $\mathcal{G}$ is
\begin{equation}
\label{eq:G_Expression}
\mathcal{G}(j_1,j_2, \ldots, j_m) = \sum_{n=1}^N \mathbf{k}_{1n}(j_1) \mathbf{k}_{2n}(j_2) \ldots \mathbf{k}_{mn}(j_m),
\end{equation}
where $\mathbf{k}_{pn}(j_p)$ is the $j_p$'th element of the vector $\mathbf{k}_{pn}, p = 1, \ldots, m$. Additionally, the $(j_1,j_2, \ldots ,j_m)$'th entry of $\mathcal{C}$ is
\begin{equation}
\mathcal{C}(j_1,j_2, \ldots, j_m) = \sum_{n=1}^N \phi_{1n}(i_1) \phi_{2n}(i_2) \ldots \phi_{mn}(i_m),
\end{equation}
where $\phi_{pn}(i_p)$ is the $i_p$'th element of the vector $\phi(x_{pn}), p = 1, \ldots, m$. According to the definition of the tensor-matrix product, we have
\begin{equation}\notag
\begin{split}
& (\mathcal{C} \times_p \phi^T(X_p))(i_1, \ldots, i_{p-1}, j_p, i_{p+1}, \ldots, i_m) \\
= & \sum_{i_p=1}^{D_p} C(i_1,i_2, \ldots, i_m) \phi_p^T(j_p,i_p) \\
= & \sum_{i_p=1}^{D_p} \sum_{n=1}^N \phi_{1n}(i_1) \phi_{2n}(i_2) \ldots \phi_{mn} (i_m) \phi_p^T(j_p,i_p) \\
= & \sum_{n=1}^N \phi_{1n}(i_1) \ldots \phi_{p-1,n}(i_{p-1}) \phi_{p+1,n}(i_{p+1}) \ldots \phi_{mn}(i_m) \sum_{i_p=1}^{D_p} \phi_{pn}(i_p) \phi_p^T(j_p,i_p) \\
= & \sum_{n=1}^N \phi_{1n}(i_1) \ldots \phi_{p-1,n}(i_{p-1}) \phi_{p+1,n}(i_{p+1}) \ldots \phi_{mn}(i_m) \mathbf{k}_{pn}(j_p)
\end{split}
\end{equation}
Then the $(j_1,j_2, \ldots, j_m)$'th entry of $\mathcal{F}$ is
\begin{equation}
\label{eq:F_Expression}
\begin{split}
\mathcal{F}(j_1,j_2, \ldots, j_m)
= & (\mathcal{C} \times_1 \phi^T(X_1) \times_2 \phi^T(X_2) \ldots \times_m \phi^T(X_m))(j_1,j_2, \ldots, j_m) \\
= & \sum_{n=1}^N \mathbf{k}_{1n}(j_1) \mathbf{k}_{2n}(j_2) \ldots \mathbf{k}_{mn}(j_m).
\end{split}
\end{equation}
By comparing (\ref{eq:G_Expression}) and (\ref{eq:F_Expression}), we complete the proof.
\end{proof}

\bibliographystyle{spbasic}      
\bibliography{./TCCA_Arxiv_Draft}   

\end{document}